\pdfoutput=1
\newif\ifscreen \screenfalse 

\newif\ifnotes \notesfalse   
\newif\ifarxiv\arxivtrue
\newif\ifconf \conffalse

\ifnotes
\newcommand{\akshay}[1]{\textsf{\color{orange} Akshay: { #1}}}
\newcommand{\note}[1]{\textsf{\color{blue} Note: { #1}}}
\newcommand{\vinod}[1]{\textsf{\color{red} Vinod: { #1}}}
\newcommand{\preetum}[1]{\textsf{\color{blue} Preetum: { #1}}}
\else
\newcommand{\akshay}[1]{}
\newcommand{\note}[1]{}
\newcommand{\vinod}[1]{}
\newcommand{\preetum}[1]{}
\fi

\ifconf
	\documentclass[final, 12pt]{colt2019} 

	\title[Computational Limitations in Robust Classification]{Computational Limitations in Robust Classification and \\ Win-Win Results}
	\usepackage{times}

	\coltauthor{%
        \Name{Akshay Degwekar} \Email{akshayd@mit.edu}\\
	    \addr Massachusetts Institute of Technology
        \AND
	    \Name{Preetum Nakkiran} \Email{preetum@cs.harvard.edu}\\
	    \addr Harvard University
        \AND
	    \Name{Vinod Vaikuntanathan} \Email{vinodv@mit.edu}\\
	    \addr Massachusetts Institute of Technology
      }
\else
	\ifscreen
		\documentclass[12pt]{article}
		\usepackage[paper=b5paper, top=0.25in, bottom=0.5in, inner=0.25in, outer=0.25in,nohead]{geometry}
	\else
		\documentclass[11pt]{article}
		\usepackage[paper=letterpaper,margin=1in,nohead]{geometry}
	\fi
	\usepackage{subcaption,boxedminipage}
	\usepackage[usenames,dvipsnames,svgnames,table]{xcolor}
	\ifarxiv
	\newcommand{\citep}[1]{\cite{#1}}
	\else
	\usepackage[hyperfootnotes=false,hypertexnames=false,colorlinks=true,linkcolor=Purple,citecolor=Purple]{hyperref}
	\usepackage{natbib}
	\fi
	
\fi

\usepackage{enumitem}
\usepackage{dsfont}
\usepackage{bbm}

\ifconf
\usepackage[]{mymacros}
\renewcommand{\set}[1]{\left\{ #1 \right\} }

\newtheorem*{lemma*}{Lemma}
\newtheorem*{claim*}{Claim}
\newtheorem*{blprexample*}{BLPR Counter-Example}
\newtheorem{assumption}[theorem]{Assumption}
\newtheorem{claim}[theorem]{Claim}
\else
\usepackage[theorems]{mymacros}
\newtheorem*{blprexample*}{BLPR Counter-Example}
\fi
\Crefname{claim}{Claim}{Claims}
\crefname{claim}{Claim}{Claims}
\Crefname{assumption}{Assumption}{Assumptions}
\crefname{assumption}{Assumption}{Assumptions}

\begin{document}
\ifconf
\maketitle
\else
	\title{Computational Limitations in Robust Classification and Win-Win Results\footnote{\akshay{added.}This work is a merge of \cite{DV19old} and \cite{nakkiran}.}} 
	\author{Akshay Degwekar
        \and Preetum Nakkiran
		\and Vinod Vaikuntanathan}
	\date{\today}
	\maketitle
	\thispagestyle{empty}
\fi

\begin{abstract}
We continue the study of statistical/computational tradeoffs in learning robust classifiers, following the recent work of Bubeck, Lee, Price and Razenshteyn who showed examples of classification tasks where (a) an efficient robust classifier exists, {\em in the small-perturbation regime}; (b) a non-robust classifier can be learned efficiently; but (c) it is computationally hard to learn a robust classifier, assuming the hardness of factoring large numbers. Indeed, the question of whether a robust classifier for their task exists in the large perturbation regime seems related to important open questions in computational number theory.

In this work, we extend their work in three directions.

First, we demonstrate classification tasks where computationally efficient
robust classification is impossible, even when computationally unbounded robust
classifiers exist. For this,
we rely on the existence of average-case hard functions,
requiring no cryptographic assumptions.
\preetum{Edited.}

Second, we show hard-to-robustly-learn classification tasks {\em in the large-perturbation regime}. Namely, we show that even though an efficient classifier that is {\em very robust} (namely, tolerant to large perturbations) exists, it is computationally hard to learn any non-trivial robust classifier. Our first construction relies on the existence of one-way functions, a minimal assumption in cryptography, and the second on the hardness of the learning parity with noise problem. In the latter setting, not only does a non-robust classifier exist, but also an efficient algorithm that generates fresh new labeled samples given access to polynomially many training examples (termed as generation by Kearns et.\ al.\ (1994)).

Third, we show that any such counterexample implies the existence of cryptographic primitives such as one-way functions or even forms of public-key encryption. This leads us to a win-win scenario: either we can quickly learn an efficient robust classifier, or we can construct new instances of popular and useful cryptographic primitives.
\end{abstract}

\ifconf

\else
	\clearpage
	\setcounter{page}{1}
	\setcounter{tocdepth}{2}
	\tableofcontents
	\pagenumbering{roman}
	\newpage
 	\pagenumbering{arabic}
\fi



\section{Introduction}

\def\err{\delta}
\def\errrob{\overline{\err}}

The basic task in learning theory is to learn a classifier given a dataset. Namely, given a labeled dataset $\{(X_i, f(X_i))\}_{i\in [n]}$ where $f$ is the unknown ground-truth and $X_i$ are drawn i.i.d. from a distribution $D$, learn a classifier $h$ so as to (approximately) minimize
\[ \err := \Pr_{X\sim D}[h(X) \neq f(X)]  \]
Adversarial machine learning is harder in that the learned classifier is required to be {\em robust}. Namely, it has to produce the right answer even under bounded perturbations (under some distance measure) of the sample $X\sim D$. That is, the goal is to learn a classifier $h$ so as to (approximately) minimize
\[ \errrob := \Pr_{X \sim D}[\exists Y \in B(X, \eps) \mbox{ s.t. } h(Y) \neq f(Y)] \]
where $ B(X, \eps) = \set{ Y :  d(X,Y) \leq \epsilon }$ and $d$ is the distance measure in question.

Learning {\em robust} classifiers is an important question given a large number of attacks against practical machine learning systems that show how to minimally perturb a sample $X$ so that classifiers output the wrong prediction with high probability. Such attacks were first discovered in the context of spam filtering and malware classification~\cite{dalvi2004adversarial,lowdM05,biggio2018wild} and more recently, following~\cite{goodfellowexplaining,szegedy2013intriguing}, in image classification, voice recognition and many other domains.

This state of affairs raises a slew of questions in learning theory.  Fix a concept class $\mathcal{F}$ and a distribution $D$ for which efficient (non-robust) learning is possible. Do there exist robust classifiers for $F$? Do there exist {\em efficiently computable} robust classifiers for $F$? Pushing the envelope further, can such classifiers be learned with small sample-complexity? and finally, is the learning algorithm computationally efficient? The answer to these questions give rise to five possible worlds of robust learning, first postulated in two recent works~\cite{BPR18} and~\cite{BLPR18}, henceforth referred to as BPR and BLPR respectively.\footnote{To be precise, \cite{BPR18} postulated four worlds, namely worlds $1$ and $3$--$5$. Subsequent work of \cite{BLPR18} added the second world.} This is the starting point of our work.

\begin{description}
  \item[World $1$.] No robust classifiers exist, regardless of computational or sample-efficiency considerations. \cite{FFF18} show a learning task in this world, namely one where computationally efficient non-robust classification is possible, no robust classifiers exist. On the other hand, for natural learning tasks, humans seem to be robust classifiers that tolerate non-zero error rate $\epsilon$, indeed even efficient robust classifiers; see \cite{BPR18} for a more detailed discussion.
  \item[World $2$.] Robust classifiers exist, but they are computationally inefficient. We demonstrate learning tasks in this world. 
 \akshay{edited. old. Learning tasks in this world were first demonstrated by \cite{nakkiran}.
  The present work is partially merged with \cite{nakkiran}, and presents other such
  learning tasks.
  \preetum{Added}}
  \item[World $3$.] Computationally efficient robust classifiers exist, but learning them incurs large sample complexity.
  \cite{schmidt2018adversarially} show a learning task where a computationally efficient robust classifier exists, but learning it requires polynomially more samples than non-robust learning. On the other hand, \cite{BPR18} show that this gap {\em cannot be} more than linear in the dimension; see \cite{BPR18} for a more detailed discussion.
  \item[World $4$.] Computationally efficient robust classifiers exist, and can be learned sample-efficiently, but training is computationally inefficient. 
  \cite{BLPR18} show a learning task in this world. However, as we observe below, their {\em computationally efficient} robust classifier only recovers from a very small {\em number} (indeed, a constant number) of perturbations. Whether there exists an efficient robust classifier for their task that recovers from large perturbations seems related to long-standing open questions in computational number theory~\cite{nadia-personal-comm,blogpost}.
  As our second result, we show two examples of learning tasks that live in this world; more details in Section~\ref{sec:results}.
  \item[World $5$.] The best world of all, in which there exist efficient algorithms both for classification and training, and the sample complexity is small (but it could be that we haven't discovered the right algorithm just yet.)

  We want to understand -- {\em are we likely to find learning tasks such as the ones \cite{BLPR18} and we demonstrate in the wild?} To that end, our third result is a {\em win-win} statement: namely, any such learning task gives rise to a cryptographic object-- either a simple one like a one-way function or a complex one like public-key encryption.
\end{description}

We proceed to describe the three results in more detail.

But before we do so, a word of warning. We and \cite{BLPR18} define these five worlds in a coarse way using polynomial-time as a proxy for computational efficiency, and a large constant accuracy as a proxy for successful classification. (We should also mention that \cite{BPR18} use SQ-learning as a different proxy for computationally efficient learning.) One could be more careful and speak of running-time/accuracy tradeoffs in the different worlds, but since our goal here is to show broad counterexamples, we do not attempt to do such a fine-grained distinction.

%
%

\section{Our Results}\label{sec:results}
We explore the relationship of computational constraints and efficient robust classification. The setting we consider consists of two distributions $ D_0, D_1$ and the classifier has to correctly classify inputs from both.
We consider the two facets to efficient robust classification:  (1) \emph{existence:} do efficient robust classifiers exist? (corresponds to World 2) and  (2) \emph{learnbility:} can we learn robust classifiers efficiently? We show three sets results on which we elaborate below.

\subsection{Existence (World 2)}
In terms of feasibility, we show that there are learning tasks where while
inefficient robust classification is possible, \emph{no efficient robust
classifiers exist}. That is, we demonstrate learning tasks in World 2.
We can show the following:
\preetum{Edited this section}

\begin{theorem}[Informal] \label{infthm:avg-case}
There exist classification tasks over  where (1) efficient non-robust classifiers exist,
(2) no efficient robust classifier exists,
but (3) inefficient robust classifiers exist.
\end{theorem}

This result does not require cryptographic assumptions, and relies only
on the existence of \emph{average-case hard} functions and good
error-correcting codes.
In fact, this result scales down to more fine-grained notions of efficiency than
polynomial-time.
All that is required is a function that is average-case hard for the
``efficient'' class, but computable by the ``inefficient'' class.

We give several examples of such learning tasks, including some examples that
require cryptographic assumptions but obtain other desirable properties
(such as obtaining tasks with efficiently-samplable distributions).
More details are given in
\cref{sec:tech:avgcase,sec:tech:codes,sec:lpn,sec:lwe,sec:avgcase-ecc}.

\subsection{Learnability (World 4)} We want to understand the hardness of learning an efficient robust classifier when it exists.
The starting point of this work was the BLPR work \citep{BLPR18}.
They showed that under cryptographic assumptions, there exists a learning task which admits efficient robust classifiers, but it is computationally infeasible to train such a classifier. More precisely, they showed that there exists a classification task (over $ \zo^{n} $) where (a) learning {\em any non-trivial} robust classifier is computationally infeasible while (b) an efficient robust classifier exists.

Unfortunately, we observe that their robust classifier is efficient only when correcting a constant number of errors. Indeed, as we explain in \cref{sec:tech:blpr}, the question of whether there exists a computationally efficient robust classifier for their task correcting even $\omega(1)$ bits of error is an important open question in computational number theory that has received some attention in the cryptanalysis community \citep{blogpost,nadia-personal-comm}.
 
The BPLR construction can be rescued using error correcting codes to enable efficient robust classifiers robust to large (constant fraction) perturbations. Our results strengthen theirs in two ways: we can weaken the required cryptographic assumption to that one-way functions exist and demonstrate tasks where the gap between learning and robust classification is more: in that efficient learning algorithms can learn to not only classify, but also to generate fresh samples from the distributions.

%

%

\begin{theorem}[Informal] \label{infthm:owfs}
Under the minimal cryptographic assumption that one-way functions, there exist classification tasks over $ \zo^m $ where (1) it is easy to learn a non-robust classifier (2) an efficient robust classifier that tolerates $ m/8 $-sized perturbations exists, and (3) it is computationally hard to learn any non-trivial robust classifier.
\end{theorem}

\begin{theorem}[Informal] \label{infthm:lwelpn}
Assuming  Learning Parity with Noise (or Learning with Errors) in the ``public-key'' regime of parameters, there exist classification tasks on $ \zo^m $ where (1) it is easy to learn a non-robust classifier. (2) an efficient robust classifier tolerating $ O(\sqrt m) $-errors exists, and (3) it is computationally hard to learn any non-trivial robust classifier.

Furthermore, it is easy to learn generators/evaluators for the non-robust distributions.\footnote{Generators and Evaluators \citep{kearns1994learnability}, are algorithms that can sample from the distribution and output the pdf of the distribution respectively.}
\end{theorem}

We elaborate on the differences between the two theorems in the techniques section. Briefly, there are three key differences: \cref{infthm:lwelpn} requires a stronger assumption, but gives a more ``natural'' example where the resulting distributions are ``more easier'' to learn non-robustly. In particular, it is easy to learn how to generate fresh samples from the two distributions, something that the one-way function based example cannot support.
This is important because we want to separate the complexity of learning the distribution from that of robust classification. And here, these distributions can be learned in a stronger sense while still being hard to classify  under adversarial perturbations.

\subsection{A Win-Win Result} Finally, we want to understand -- \emph{Are we likely to find such learning tasks in the wild?} To that end, we show a converse to our results. Namely,

\begin{theorem}[Informal] \label{infthm:converse}
Any computational task where an efficient robust classifier exists, but is hard to learn one in polynomial time implies one-way functions, and hence symmetric key cryptography.
\end{theorem}
Furthermore, if the learning task satisfies certain natural properties, it gives us (a certain weaker form of) public-key cryptography as well!

It would be very surprising to us if public-key cryptography (and even one-way functions) arise out of natural classification tasks on, say, images. Thus, perhaps uncharacteristically for cryptographers, we offer a possible (optimistic) interpretation of this state of affairs: namely, that for {\em natural} learning tasks where there exists a robust classifer, it can also be efficiently found, we just haven't figured out the right algorithm yet.

An important caveat is due here: our definition of hardness of learning a robust classifier is a strong one: it requires that the perturbing adversary be constructive and universal. Our classification tasks do satisfy this definition, and that only makes them stronger. On the other hand, it does make our converse weaker. More details are given in \cref{sec:tech}.

\subsection{Related Work.}
\akshay{this needs work. Perhaps should go after techniques.  }
The works closest to ours are \cite{BPR18, BLPR18}. We discuss them last.

\subparagraph{Adversarial Examples.} The problem of adversarial classification ws first considered by \cite{dalvi2004adversarial}. Starting with \cite{szegedy2013intriguing}, there is a large body of work demonstrating the existence of small adversarial perturbations in neural networks that cause them to misclassify examples with high confidence. There have been various approaches proposed against such perturbations and many of them have been broken (see \cite{CW17, ACW18} and references therein).

A line of work which \citep{gilmer2018adversarial, FFF18, mahloujifar19can} shows that for certain learning tasks and distributions (eg spheres in $ \RR^n $ or product distributions), due to concentration of measure adversarial examples exist close to points in the distribution and can at times be found efficiently for classifiers that are not perfectly correct, pointing to the challenges of robust classification in this setting. These works show evidence for World 1: that for certain specific models and training algorithms, robust classifiers don't exist.
In our learning tasks, robust classification is possible, albeit computationally inefficient.

\cite{schmidt2018adversarially} demonstrate simple classification tasks (distinguishing between high dimensional gaussians) where the sample complexity of robust learning is higher than that of classical learning by a polynomial factor. Hence they show evidence for world 3. \cite{BPR18} show that this gap is essentially tight.
This work is similar in spirit to ours, with the resource being sample complexity instead of computational complexity in our case. In the case of computational complexity, we can essentially show exponential gap between the running time required for learning non-robustly vs learning robust classifiers.

%

\subparagraph{BPR/BLPR \citep{BPR18, BLPR18}.} \akshay{TODO: Rewrite this. }
In BPR, they showed two results. First, that in the world of polynomial sample complexity with no bounds on running time, learning a non-robust classifier and learning a robust classifer have the comparable sample complexity, if such a robust classifier exists. Second, they exhibit a learning task where while learning a robust classifier was information-theoretically easy with polynomial sample complexity, but doing so was difficult in the SQ model and it required exponentially many queries. This gives rise to a task where learning a robust classifier in a computationally efficient manner (in the SQ model) was a lot harder than doing so inefficiently.

In a followup work, BLPR they considered strengthening the second BPR result to show that under cryptographic assumptions, there exists a learning task which admitted efficient robust classifiers, but it was computationally infeasible to do so. They showed that there exists a classification task (over $ \zo^{n} $) where learning any non-trivial robust classifier is computationally infeasible while an efficient robust classifier exists that can correct $ O(1) $-bit error.
A description of their construction is given in \cref{sec:tech:blpr,sec:bpr-bbs}.

\akshay{Removed: 
\subparagraph{\citep{nakkiran}.}
\preetum{Added.}
This work introduces learning problems in World 2. The primary focus of
\cite{nakkiran} is on learning problems over $\mathbb{R}^n$, and
measures the magnitude of adversarial perturbation using the $\ell_\infty$-norm
over $\mathbb{R}^n$.
In our work, we consider learning problems over $\zo^n$, and measure
perturbations in Hamming distance.
Nevertheless, similar techniques are involved, and this paper is partially
merged with \cite{nakkiran}.
Moreover, \cite{nakkiran} explores classification in weaker models than
polynomial-time, and demonstrates a tradeoff between adversarial-robustness and accuracy for
some weaker models (e.g. linear classifiers).
In contrast, in this work we are primarily concerned with polynomial-time
classifiers.|
}

\section{Our Techniques}\label{sec:tech}
In this section, we give a high level description of our techniques. We begin by describing the BLPR classification task and its limitations. Then we describe the definition of robust classification and non-existence/unlearnability of such classifiers. We then describe several recipes for constructing tasks where robust classification is computationally intractable. \akshay{Moved it ahead. I was going to move it to the first, but the "BPR trick" attribution etc makes it simpler to move it second.}
In the first recipe, based on one-way functions, we show tasks where while efficient robust classifiers exist, but are hard to learn, thus proving \cref{infthm:owfs}.
The second recipe assuming average-case hard functions proves \cref{infthm:avg-case}, where no efficient robust classifiers exist. 
The final recipe is based on hardness assumptions on decoding noisy codewords /
lattices, namely Learning Parity with Noise (LPN) and Learning with Errors (LWE)
and proves \cref{infthm:avg-case,infthm:lwelpn} in different parameter regimes.

\subsection{The BLPR Classification Task}\label{sec:tech:blpr}
We sketch the \cite{BLPR18} classification task where it is difficult to learn a robust classifier. A more detailed description of their construction is given in Appendix \ref{sec:bpr-bbs}.

The key object in their construction is a ``trapdoor pseudorandom generator''. A pseudorandom generator $ \prg: \zo^n \rightarrow \zo^{2n}$ is an expanding function whose outputs are indistinguishable from truly random strings. That is, $ \set{\PRG(x) : x \gets \zo^n} \approx_c \set{y: y\gets \zo^{2n}} $.\footnote{We say that two families of distributions $ \set{X_n}_{n\in \NN} $ an $ \set{Y_n}_{n\in \NN} $ are \textdef{computationally indistinguishable} (denoted by $ \set{X_n} \approx_c \set{Y_n} $ or  $ X \approx_c Y $ for brevity) if for every polynomial time distinguisher $ \sfD $,
	$ |\prob[x\gets X_n]{\sfD(x)} - \prob[y\gets Y_n]{\sfD(y)} |\leq \negl(n) $. } A trapdoor pseudorandom generator has a hard-to-find trapdoor $ \state $ that allows distinguishing the output of the PRG from random outputs. That is, there exists a distinguisher $ \sfD $ such that (say),
\[ \prob[x \gets \zo^n]{\sfD(\state, \tprg(x)) = 1} - \prob[y \gets \zo^{2n}]{\sfD(\state, y) = 1} > 0.99\]
They show that the Blum-Blum-Shub Pseudorandom generator \citep{BBS86} has such a trapdoor.
Given a trapdoor PRG, their learning task $ D_0, D_1 $ is the following:
\[ D_0 = \set{(0,\tprg(x)) : x \gets \zo^n} \text{ and, } D_1= \set{(1, y) : y\gets \zo^{2n}} \ .\]
The first bit enables easy non-robust classification. The fact that there exists an inefficient robust classifier follows from a volume argument -- that the there are a few PRG outputs in a large domain. This implies that there is an inefficient robust classifier that tolerates $ O(\sqrt{n}) $-sized perturbations.
That a robust classifier is hard to learn follows from the perturbing adversary that sets the first bit to $ 0 $. A robust classifier has to distinguish between outputs of the PRG from random strings, without the trapdoor. This is infeasible by the security guarantee of the PRG.

Finally, what needs to be proved is that the trapdoor enables {\em robust classification}. The trapdoor indeed does enable a robust classifier that tolerates constant-sized pertubations (i.e., if any constant number of bits are altered) simply by exhaustive search among the polynomially many possible sets of perturbed bits. For a constant $ c $, the robust classifier given input $ y $ goes over all $ n^c $ words in the Hamming ball $y' \in B( y, c) $ and checks if the distinguisher $ \sfD(\state, y') = 1 $. If yes, output $ D_0 $ else output $ D_1 $. But this approach does not give a classifier beyond constant-sized errors because the running time is exponential in the number of errors corrected.

The primary limitation of trapdoor PRGs is that the trapdoor does not enable decoding the PRG output from the perturbed samples, only distinguishes PRG outputs from random strings.
Indeed, for the Blum-Blum-Shub trapdoor PRG (and related constructions such as the one of Micali and Schnorr~\citep{MS}) considered in BLPR, the question of whether there is any trapdoor that permits robust inversion is an open question in computational number theory. We refer the reader to \cref{sec:bpr-bbs} for discussions regarding related questions.
\akshay{added} To enable efficient decoding, their construction can be modified by using an error correcting code to make it robust to larger pertubations.

\subsection{Definitions: Robust Classification}
We start by describing the notion of robust classification and hardness of robust classification used.

\paragraph{Robust Classifier.} When we state that \emph{a robust classifier} exists (for given $ \eps $), we show the strongest notion: that there exists a classifier $ \sfR $ (efficient or inefficient, as specified) that classifies all input close to a random sample correctly:
\[\text{ For $ b \in \zo $, } \prob[x \gets D_b]{\sfR(x') = b \text{ for all }x'\in B(x, \eps)} > 0.99 \ .  \]

\paragraph{Non-Existence/Unlearnability of Robust Classifiers.}
When we describe \emph{the non-existence (or unlearnability) of robust classification}, we satisfy the strongest notion: that there exists a poly-time perturbation adversary $ \sfP $ whose perturbed examples cannot be classified better than chance by any efficient (or efficiently learned) classifier. That is, for any efficient $ \sfR $ (or $ \sfR \gets \learn^{D_0,D_1}(1^n) $),
\[ \prob[x\gets D_b]{\sfR(\sfP^{D_0, D_1}(x)) = b} < 0.5  + \negl(n) \ ,\]
where a $ \negl: \NN \rightarrow \RR $ is a function such that $ \negl(n)  < n^{-c}$ for all $ c\in \NN $ for large enough $ n $.  
See Definition \ref{def:robust_hardness} for a formal definition of hard to learn robustly. This definition has two key properties: it is \emph{constructive} (adversarial perturbations are found) and \emph{universal} (the same adversary works for all algorithms).

The negation of the robust classification definition suggests the following definition: for efficiently learned classifiers $ \sfR $, $ \prob[x\gets \cD_b]{B(x,\eps)\not\subseteq \sfR^{-1}(b)} < 0.99 $.
This definition is unsatisfying because it says nothing about the hardness of finding such misclassified examples. In particular, if such adversarial perturbations existed but were computationally hard to find, then the existence of adversarial examples is not an issue. Hence, we choose a constructive definition that requires such examples to be efficiently found. The fact that the adversary is universal only makes the counter examples stronger.


\subsection{Unlearnability From Pseudorandom Functions and Error Correcting Codes}
In this section, we construct a learning task where classification is easy, robust classifier exits, but is hard to learn. The primary ingredients of this construction are pseudorandom functions and error correcting codes. We introduce both the primitives and build the construction in stages.
A pseudorandom function family (PRF) \citep{GGM86} is a family of keyed functions $ F_k: \zo^n \rightarrow \zo $ where the key $ k \gets \zo^n $, that are indistinguishable from uniformly random functions to any polynomial time algorithm. That is, for every poly time algorithm $ \sfA $,
\[ \prob[k\gets \zo^n]{\sfA^{F_k}(1^n)} \approx \prob[U_n]{A^{U_n}(1^n)} \]
where $ U_n : \zo^n \rightarrow \zo $ is a uniformly random function. PRFs can be constructed from one-way fucntions.
Kearns and Valiant \citep{KV94} constructed a hard to learn classification task using pseudorandom functions as follows:
\[ D_0 = (x, F_k(x)) \text{ and, } D_1 = (x, 1- {F_k(x)})\]
The task essentially asks to efficiently learn a predictor for the pseudorandom function which is difficult. To transform this task to one that is hard to learn robustly, while an efficient robust classifier exists, we use error correcting codes.
Recall that an error correcting code has two algorithms $ (\encode, \decode) $ where $ \encode $ returns a redundant encoding of the message that the $ \decode $ algorithm can efficiently recovers the encoded message even when the encoded codeword is tampered adversarially to some degree. So, consider the following classification task: distinguish between error-corrected versions of the PRF:
\[ D_0 = \encode(x, F_k(x)) \text{ and, } D_1 = \encode(x, 1- {F_k(x)}) \ .\]
Note that this task has the following properties: (1) A robust classifier exits and, (2) a robust classifier is hard to learn. For the first property, consider the following robust classifier: the classifier given the secret key, first decodes the perturbed sample using the $ \decode $ algorithm and then checks if is of the form $ (x, F_k(x)) $ or $ (x, 1 - {F_k(x)}) $ and outputs which case it is. The robustness follows from the error correcting code. The fact that no classifier is learnable follows from the fact that the PRF is hard to predict, and thats exactly what the classifier has to do.
Finally, we want the task to be easy to classify non-robustly. Here we use the ``BPR trick'' (\cite{BPR18}). That is, we additionally append to each sample a bit indicating which distribution it was sampled from. That is,
\[ D_0 = (0, \encode(x, F_k(x))) \text{ and, } D_1 = (1, \encode(x, 1- {F_k(x)})) \ .\]
Now the samples are easy to classify non-robustly, simply output the first bit. Learning a robust classifier is hard, for that, consider the perturbing adversary that erases the first bit. For these samples, robust classification is identical to predicting the output of the PRF. This is difficult for any efficiently learned classifier. Hence, this gives us a task that that is easy to classify, has an efficient robust classifer and yet, any non-trivial robust classifier is hard to learn.

Note that because we have excellent error correcting codes, this recipe is maximally robust. We can pick a code that tolerates a constant fraction ($ \frac14 - \eps $) errors and still enable correct decryption \citep{GI01}. This can be further boosted to $ (\frac 12 - \eps ) $ by using list decoding instead of unique decoding and increasing the output size of the PRF to $ n $-bits. 
 We do not formally write this construction.  


\subsection{Non-Existence of Robust Classifiers from Average-Case Hardness}
\label{sec:tech:avgcase}
\preetum{Added.}\akshay{Edited the title, and blurb below. Moved the section ahead.}
This section describes a learning task for which no computationally efficient robust classifier exists, even though inefficient ones do, based on average-case hard functions, thus proving \cref{infthm:avg-case}. 

Let $g: \{0, 1\}^n \to \{0, 1\}$ be a function that is \emph{average-case hard},
such that no polynomial-time nonuniform algorithm can compute $z \mapsto g(z)$
noticeably better than random guessing.
For example, taking $g$ to be a random function $\{0, 1\}^n \to \{0, 1\}$
suffices.
Let $(\encode, \decode)$ be a good error correcting code, capable of decoding
from a constant fraction of errors.
Now, construct distributions $D_0, D_1$ as follows:
$$
D_0 = (0, \encode(x, g(x)))
\text{ and }
D_1 = (1, \encode(x, 1-g(x)))
$$
for $x \gets \{0, 1\}^n$ uniformly.
Note that these distributions are trivially distinguishable non-robustly.
However, with a perturbation adversary that destroys the first coordinate,
distinguishing $D_0$ from $D_1$ essentially requires computing the function $g(x)$,
which cannot be done efficiently. Thus, there is no efficient robust
classifier.
Moreover, an inefficient robust classifier exists, since one can decode the
error correcting code (correcting any adversarial errors) and compute $g(x)$. 

\akshay{Added.}
When using an average-case hard function, one limitation here is that the algorithm generating the samples from distributions $ D_0, D_1 $ is inefficient. This can be remedied by using one-way functions, because generating $ D_0, D_1 $ requires the algorithm to perform the simpler task of sampling $ (z, g(z)) $ for random $ z $'s, and not computing $ g(z) $ given $ z $, that the classifier has to do. In fact, this is precisely the difference between average-case $ \NP $ hardness, which requires us to generate hard instances, and one-way functions, which require generating hard instances along with their solutions. 
See \cref{sec:tech:avgcase} for more details.


\subsection{From Hardness of Decoding under Noise.}
\label{sec:tech:codes}
This section describes a proof sketch for \cref{infthm:avg-case,infthm:lwelpn}.
The problems Learning Parity with Noise (LPN) and Learning with Errors (LWE) have the following flavor: In both the problems a random code $ C $ (over $ \ZZ[2] $ in LPN, $ \Zq $ for a large prime in LWE) is specified by a matrix $ \mat A $:
\[ C = \{\vec s^T \mat A \} \text{ or, the dual form } C = \{\vecy : \mat A \vec y = \vec 0\} \ .  \]
Then the computational task is to distinguish a point close to the code from a uniformly random point in the space. The conjectured hardness of these problems can be used to construct a variety of cryptographic primitives. In the overview, we will describe the construction with the LPN assumption. The LWE construction is conceptually identical.
\paragraph{The Classification Task.} We begin by describing the classification task and then the rationale. The  task consists of two distributions on samples $ D_0, D_1 $ picked as follows: Pick a random linear code over $ \ZZ[2] $, $ C: \zo^{n} \rightarrow \zo^{8n} $, (described by the generator matrix $ \mat A $ or the parity check matrix $ \mat H $). Then,
\[ D_0 = \set{\vec y : \vec y\gets C} \text{ and, } D_1 = \set{\vec y + \vec 1 : \vec y \gets C} \ . \]
So, the task is to distinguish codewords of $ C $ from their affine shift ($ \vec 1 $ represents the all-ones vector).  The distributions are easy to classify non-robustly.  There exists an inefficient robust classifier because the distance between the two codes $ C $ and $ C + \vec 1 $ is large.

To show that a robust classifier is hard to learn,  consider
the perturbation adversary that adds random noise of varying size to the two distributions. Learning a robust classifier for this adversary is equivalent to distinguishing LPN samples from random. Hence any computationally efficient adversary cannot classify these examples better than chance.

Finally, we need to show that for a certain perturbation regime, no efficient robust classifier exists while for a different perturbation regime, an efficient robust classifier does exist. The latter is accomplished by the notion of ``trapdoor sampling'' where the code is sampled with a trapdoor that enables decoding noisy codewords (and hence robust classification too).

Below we describe the example in more detail and give a sketch of the arguments needed. Formal proofs are given in \cref{sec:lpn,sec:lwe}.
\paragraph{LPN Assumption.} The LPN hardness assumption states that: for $ m = \poly(n) $,
\[ (\mat A, \vec s^T \mat A + \vec e^T) \bmod 2 \approx_c (\mat A , \vec r) \bmod 2 \]
where $\mat A \gets \ZZ[2]^{n\times m} $ describes a random code, the secret $ \vec s \gets \ZZ[2]^n$ is drawn unifomly at random and each coordinate of error $ \vec e \gets \Ber(r)^m $ drawn from a Bernoulli distribution with error rate $ r $, i.e., probability of drawing $ 1 $ is $ r $.

\paragraph{Hardness Regimes and Trapdoors.}
The key parameter which controls the hardness of LPN is the distance of the close point from the code in the appropriate norm.\footnote{The specific norm is not crucial for the discussion below. Hamming is used for LPN while for LWE, the norm is obtained by embedding $ \Zq $ in $ \ZZ[] $ as $ \{-\floor{q/2}, \dots 0, 1, \dots \floor{q/2}\} $ and take the $ \ell_\infty $ norm on $ \ZZ[] $.} As the distance increases, the problem becomes harder. In the case of LPN, this distance is approximately $ m\cdot r $ where $ r $ is the error rate.
For most non-trivial parameter settings of the distance parameter, these two problems are \emph{believed to be computationally intractable}. That is, an efficient algorithm given a description of the random code cannot distinguish random points close to the code from random points in the space.

Along with their conjectured computational hardness, we are interested in another property of these problems, \emph{the existence of a trapdoor}: that is, can we sample the code along with some polynomial-size side information that lets us distinguish efficiently random points from points close to the code. This information usually is a ``short basis'' for the dual code.
The trapdoor property has two important regimes: the ``public-key'' regime and the ``private-key'' regime. In the case of LPN, the public-key regime corresponds to error rate $ r = O(1/\sqrt{n}) $ while the private-key regime translates to constant error rates, e.g., $ r = 0.1 $.
The public key regime of parameters enables construction of advanced cryptographic primitives, including public key encryption. On the other hand, in the private-key regime, we know constructions of one-way functions and symmetric key cryptography, but not much more.

Importantly for us, in ``public-key'' parameter regime, such a trapdoors exists and can be sampled efficiently. On the other hand,  in the private-key regime, it is conjectured that no such trapdoor exists.
Traditionally this problem is studied as the problem of \emph{decoding linear codes with preprocessing} (for LPN) and \emph{closest vector problem with preprocessing} (for LWE). In the problem of decoding linear codes with preprocessing, an inefficient algorithm $ \preprocess $ performs arbitrary preprocessing on the given linear code (described by the matrix $ \mat A $) and has to come up with a short polynomial-sized trapdoor for the code. Later the $ \decode $ algorithm has to use this trapdoor to efficiently find the codeword close to a given input. This problem and the closest vector problem (is the same problem, on lattices instead of codes)  are $ \NP $-hard to approximate in the worst-case  \citep{BruckN90,micciancio2001hardness, Regev04}.

We require an average-case variant of the problem termed as the hardness of \emph{LPN with Preprocessing}. 
The assumption is stated more formally in Assumption \ref{assn:lpnp}. This assumption can be used to construct a task where no efficient robust classifier exists. The task is very similar to the one below where a efficient robust classifier exists but is hard to find, except with higher levels of noise. More details in \cref{sec:lpn}. 

\paragraph{Task with an Hard-to-Learn Efficient Robust Classifier.} We now turn to the problem of constructing learning tasks where an efficient robust classifier exists, but is hard to learn. 
It consists of distinguishing between codewords ($ D_0 = \{\vec s \mat A\} $) from an affine shift of the codewords ($ D_1 = \{\vec s \mat A + \vec 1 \} $ where $ \vec 1 $ is the all-ones vector). That is, for a random matrix $ \mat A \in \ZZ[2]^{n \times m} $ where $ m = 8n $,
\[ D_0  =  \{\vec s^T \mat A : \vec s \in \zo^ n\} \text{ and, } D_1  =  \{\vec s^T \mat A + \vec 1: \vec s \in \zo^ n\} \]
We want to show that this task exhibits an efficient robust classifier.
For that, we need access to a trapdoor. In the case of LWE, such algorithms are known \citep{GPV08} and proven to be exremely fruitful (see \cite{Pei16} and references therein). This LWE trapdoor is a zero-one matrix $ \mat T \subseteq \zo^{m \times m-n} $ over $ \Zq $ such that $ \mat A \mat T = \vec 0 \pmod q$. Because $ \mat T $ is a zero-one matrix, given any adversarially perturbed sample $\tilde{\vec y} =  \vec s^T \mat A + \vec e^T $, multiplying by $ \mat T  $ results in $ \tilde{\vec y}  \mat T = \vec e^T \mat T $ which is a vector with small entries in each coordinate. And this can be checked to distinguish LWE samples from random.

In the case of LPN, we don't know how to perform such trapdoor sampling: where a uniformly random matrix $ \mat A $ is sampled along with such a trapdoor. Instead we rely on a computational variant of this. We can sample a matrix $ \mat H \in \ZZ[2]^{n\times 8n}$ that is indistinguishable from a random matrix along with such a ``short'' trapdoor: a matrix $ \mat E $ where each row and colum of $ \mat E$ has hamming weight $ O(\sqrt m) $. See \cref{lem:lpntrapdoor} for more details. This then allows for a similar construction. The perturbing adversary $ \sfP $ again adds random noise, this time of a lower magnitude though.
\[ \sfP(\vec y): \text{Output $ \tilde{\vec y} = \vec y + \vec e$ where $ \vec e \gets \ZZ[2, \ham = 0.1\sqrt m]^m $} \ . \]
Note that earlier, we added $ 0.1 m$ bits of noise, instead here we are adding $ 0.1 \sqrt m $ bits. This level of noise places the problem in the ``public-key'' regime of parameters. Furthermore, given the trapdoor, in this case, we can recover which distribution the unperturbed sample was sampled from, giving us the required robust classifier. See \cref{sec:lpn} for more details.

Again, it is clear that learning a non-robust classifier is easy. The hardness of LPN assumption implies that it is hard to learn a robust classifier. This is in contrast to the previous construction where no efficient robust existed. Here, the trapdoor gives us an efficient robust classifier, but the hardness of LPN implies that such a classifier is hard to learn. In fact, any efficiently learned classifier cannot do better than chance.

One thing to note is that this efficient robust classifier is not ``maximally robust'', meaning that while an inefficient robust classifier can tolerate $ 0.1m $ bits of noise and still classify correctly, the efficient classifier can tolerate $ 0.1\sqrt m $ bits of noise. This is similar in the case of LWE as well, where there is a gap between noise the trapdoor can support (about $ q/m $ in the $ \ell_\infty $ norm) against the maximally robust limit ($ \Omega(q)$). This is not surprising because decoding random linear codes is harder than decoding specifically designed codes and hence the trapdoors do not achieve optimal decoding.

A feature of this construction is that an efficient algorithm can learn to not only distinguish the samples from distributions  $ D_0 $ and $ D_1 $, it can easily learn to generate samples from the two distributions as well.

\paragraph{Comparing Recipes.}
\preetum{Edited this section. Probably needs more edits.}\akshay{Made changes.}
There are three key differences between the recipes.
The first difference is in the underlying hardness assumption.
The first two constructions are based on weaker assumptions: namely general assumptions that one-way functions exist (or average-case hard function respectively) rather than the specific assumptions of LWE and LPN.

The second difference is that the distributions based on LWE/LPN facilitate learning in a stronger sense, that it is possible to sample from the non-robust distributions after seeing polynomially many samples.  In construction I based on one-way functions, we do not learn either $ D_0 $ or $ D_1 $ in that strong sense.
In fact, after seeing polynomially many samples, efficient sampling algorithms have no non-trivial advantage with the other recipes. As pointed out in in construction II, it is possible to support generation, albeit using a slightly stronger assumption that one-way functions exist. 

The third difference is that of naturalness: we feel that the LWE/LPN recipe gives a more natural learning task. This is obviously a subjective notion. This learning task of distinguishing noisy codewords from random has existed independent of the notion of robust classification and arises naturally in other contexts. 


\subsection{Converse: Cryptography from Hardness of Robust Classification.}
\label{sec:tech:converse}

In this section, we describe how \cref{infthm:converse} is proved. 
The key result we rely on here is that we can construct one-way functions from any pair of samplable distributions that are statistically far and computationally indistinguishable.

\begin{theorem} 
Given a pair of distributions $ (X_0, X_1) \gets \cF $ over $ \cX $ that are statistically far,  i.e., 
$ d_{TV}(X_0, X_1)  > 0.9 $ 
and computationally indistinguishable. That is for every polynomial time adversary $ \sfA $ that gets sample access to the distributions,
\[  \EE_{x\gets X_0}\sfA^{(X_0, X_1)}(x) - \EE_{x\gets X_1}\sfA^{(X_0, X_1)}(x) < 0.1 \]
Then one-way functions exist.\footnote{The constants in the equations are fairly arbitrary. We can replace them by any constants $ \alpha, \beta $ where $ \alpha  > \beta $ and the result holds (see \cite{NR06,BDRV19js}).}
\end{theorem}

In order to construct such distributions, we rely on the learning task (given by $ (D_0, D_1) $) and the perturbation adversary $ \sfP $. The distributions we consider are
$$X_0 = \set{\sfP(x) : x \gets D_0} \text{ and, }  X_1 = \set{\sfP(x) : x\gets D_1} \ .$$
Note that because efficient robust classifiers are hard to learn,  no efficient algorithm $ \sfA $ (that knows $ \sfP $ and gets access to the distributions $ D_0,D_1 $) can distinguish between the two distributions . On the other hand, because a robust classifier exists, these two distributions are statistically far from each other. This implies that one-way functions exist.

\ifconf
\paragraph{Acknowledgments.} We would like to thank Shafi Goldwasser and Nadia Heninger for discussions regarding inversion of the (noisy) BBS PRG.

	\pagebreak
	\bibliography{bib}
	\appendix
\fi

\section{Definitions}

We use lowercase letters for values, uppercase for random variables, uppercase calligraphic letters (e.g., $\mathcal{U}$) to denote sets, boldface for vectors (e.g., $\vec{x}$), and uppercase sans-serif (e.g., $\algfont{A}$) for algorithms (i.e., Turing Machines). We let $\poly$ denote the set all polynomials.  A function $\nu \colon \naturals \to [0,1]$ is \textit{negligible}, denoted $\nu(n) = \negl(n)$, if $\nu(n)<1/p(n)$ for every $p\in\poly$ and large enough $n$.
Given a random variable $X$, we write $x\gets X$ to indicate that $x$ is selected according to $X$. Similarly, given a finite
set $\mathcal{S}$, we let $s\gets \mathcal{S}$ denote that $s$ is selected according to the uniform distribution on $\mathcal{S}$. For an algorithm $ \sfA $, we denote by $ x \gets \sfA $ the experiment where $ x $ is sampled by feeding a uniformly random input to $ \sfA $ from its input domain. 
We say that two families of distributions $ \set{X_n}_{n\in \NN} $ an $ \set{Y_n}_{n\in \NN} $ are \textdef{computationally indistinguishable} (denoted by $ \set{X_n} \approx_c \set{Y_n} $ or  $ X \approx_c Y $ for brevity) if for every polynomial time distinguisher $ \sfD $, 
$$ |\prob[x\gets X_n]{\sfD(x)} - \prob[y\gets Y_n]{\sfD(y)} |\leq \negl(n) $$

%
%
%
%
\subsection{Learning \& Classification}
\begin{definition}[Classification] For a family of classification tasks $ \cF $ over $ \cX $ is easy to classify if there exists a learning algorithm that given  $ \poly(n) $ i.i.d.\ samples from a pair of distributions $ (D_0, D_1 ) \in \cF$ supported on $ \cX $, outputs an efficiently computable classifier $ \sfA : \cX \rightarrow \zo $ such that, 
	$$
	\prob[X\gets D_b]{\sfA(x) = b} \geq 0.99
	$$
\end{definition}

We want to consider other notions of learning distributions as well, in order to make more refined distinctions between learning distributions. 
The following definition for learnability of discrete distributions is   from \cite{kearns1994learnability}.  
\begin{definition} For a distribution $ D $ over a discrete domain $ \cX $, 
\begin{enumerate}
\item \textbf{Generator.} A circuit $ \sfG:\zo^m \rightarrow \cX $ is an \textdef{$ \eps $-good generator} for $ D $ if 
\[ \KL(D \| \sfG(U)) \leq \eps \] 
where $ \sfG(U) $ denotes the distribution obtained by evaluating $ \sfG $ on a uniformly random input. 
\item \textbf{Evaluator. } A circuit $ \sfE: \cX \rightarrow \RR_{\geq 0} $ is an \textdef{$ \eps $-good evaluator} for $ D $ if 
\[ \KL(D \| \sfE) \leq \eps \] 
where $ \sfE $ denotes the distribution obtained by sampling with probability density function $ \sfE $. 
\end{enumerate}
\end{definition}

\begin{definition} A class of distributions $ \cF = \set{F_n} $ over a discrete domain $ \cX = \set{ \cX_n} $ is \textdef{$(\eps,\delta)$-efficiently learnable with a generator (or evaluator resp.)}
if there exists a polynomial time algorithm $ \gen $ that given oracle access to any $ D_n \in \cF_n $ runs in time $ \poly(n, 1/\delta, 1/\eps) $ and outputs $ \sfG $ (or $ \sfE $ resp.) such that with probability $ \geq 1-\delta $ over the randomness of $ \gen $ and samples, $ \sfG $ ($ \sfE $ resp.) is an $ \eps $-good generator (evaluator resp.) of $ D $. 
\end{definition}

In our examples, we seek to find distributions where the gap between ease of learning the actual distributions and that of the adversarially perturbed distributions is maximized.

\subsection{Hardness of Efficient Robust Classification}
We start by recalling the notion of robust classification. Then, 
we consider two ways of formalizing the difficulty of efficient robust classification: (1) no \emph{efficiently computable} robust classifier exists, (2) an efficient robust classifer exists, but \emph{it is hard to learn one efficiently}. 

\begin{definition} 
Consider a classification task given by two distributions $ D_0, D_1 $ over $ \cX^n $. Let $ \|\cdot \| $ be a norm over the space $ \cX^n $ and $ \eps > 0 $. Let $ \sfR: \cX^n \rightarrow \zo $ be a classifier. The classifier $ \sfR$ is $ \eps $-\textdef{robust} if 	
\begin{equation*}
	\prob[X\gets D_b]{\sfR(\tilde x) = b \text{ for all $ \tilde x \in B(x, \eps) $}} \geq 0.99
\end{equation*}
\end{definition}

In the definition above, $ B(x, \eps) $ is all points that are $ \eps  $ distance from $ x $ in the given norm. Here, we will generally be concerned with  Hamming distance and the $ \ell_1 $ norm. 
%

\begin{definition}[Hardness of Robust Classification] \label{def:robust_hardness}
Consider a family of classification tasks, defined by two distributions $ D_0 $, $ D_1 $ over $ \cX^n $ sampled from a distribution over learning tasks $ \sf{Samp} $. Let $ \|\cdot \| $ be a norm over the space $ \cX^n $ and $ \eps > 0 $. We consider the following notions of difficulty of robust classification: 
\begin{enumerate}
\item \textbf{No efficient $ \eps $-robust classifier exists.}  
There exists a polynomial-sized perturbation algorithm $ \sfP $, such that for every polynomial sized classifier $ \sfR $, the perturbed samples are hard to classify. That is,  
\[ \prob[ ]{ \sfR^{D_0, D_1}(\tilde x) = b} \leq \frac12 + \negl(n)
\]
where the perturbed sample $ \tilde x $ is generated by sampling  $x \gets D_b \text{ for a random }b\gets \zo $ and is then perturbing $ \tilde x \gets \sfP^{D_0, D_1}(x)$. 

\item \textbf{Efficient $ \eps $-robust classifier is hard to learn.} There exists a polynomial-sized perturbation algorithm $ \sfP $, such that every polynomial-time learning algorithm $ \learn $ that outputs a polynomial sized classifier $ \sfR $, the perturbed samples are hard to classify for $ \sfR $. That is,  for a learning task $ D_0, D_1 $ sampled by $ \samp $ and robust classifer $ \sfR \gets \learn^{D_0, D_1}(1^n) $ output by $ \learn $, 
\[ 	\prob[ ]{ \sfR(\tilde x) = b} \leq \frac12 + \negl(n)
\]
where the perturbed sample $ \tilde x $ is generated by sampling $x \gets D_b \text{ for a random }b\gets \zo $ and is then perturbing $ \tilde x \gets \sfP^{D_0, D_1}(x)$. The probability is over the entire experiment from sampling the learning tasks to the randomness of the perturbation algorithm and the classifier. 
\end{enumerate}	
\end{definition}

\paragraph{Discussion.} 
An alternate definition of hard to classify robustly would be the negation of robust classification. That  definition takes a following form:  
$$
\prob[x \gets D_b]{\exists \tilde x \in B(x, \eps) \text{ such that, } \sfR(\tilde x) \neq b} \geq 0.5
$$	
This definition is unsatisfactory because it does not say anything about how difficult it is to find such perturbations. In the event when such examples are not efficiently discoverable, we do not have to worry about these. 

In the definitions used, the perturbing adversary is both efficient and universal. Efficiency is a very natural property to have, in that if the adversarial examples are computationally hard to find, then they are less of a concern. The universality property says that there is a single perturbation adversary that succeeds against all efficient classifiers. This is a strong requirement. This makes our robustly hard to learn tasks better: that they have a unique perturbation adversary that is independent of which classification algorithm is used. On the other hand, it makes our converse  results constructing one-way functions from hard to learn robust tasks weaker, because they only hold for such robustly hard to learn tasks, with universal perturbation adversaries. 

It is possible to have a perturbation adversary $ \sfP $ that is efficient but not universal. The perturbation adversary gets oracle access to the classifier and has to then output a misclassified example. This is a weaker requirement than \cref{def:robust_hardness}. 
We do not know if such a definition also implies cryptography.

{
	\renewcommand{\trapsamp}{\mathsf{LPNTrapSamp}}
\newcommand{\matstack}[1]{\begin{bmatrix}#1\end{bmatrix}}
\newcommand{\leftspan}{\operatorname{rowspan}}

\section{Learning Parity with Noise}\label{sec:lpn}

\subsection{Assumption Definition and Discussion}
Let $\ZZ[2, \sf{Ham} = t]^{m} $ denote vectors in $ \ZZ[2]^m $ with Hamming weight exactly $ t $. We will consider  Hamming weight as our norm in this setting. 
\begin{definition}[Learning Parity with Noise Problem (LPN)]\label{def:lpn}
	For $ n,m,t \in \NN $, an LPN sample is obtained  by sampling a matrix  $ \mat A \gets \ZZ[2]^{n \times m} $, a secret $ \vec{s} \gets \ZZ[2]^n $, and an error vector $ \vec{\epsilon} \in \ZZ[2, \sf{Ham} = t]^{m} $ and outputting $ (\mat A, \vec s^T \mat A + \vec e^T) $.  
	
	We say that an algorithm solves $ \mathsf{LPN}_{n,m,t} $ if it distinguishes an LPN sample from a random sample distributed as $\ZZ[2]^{n\times m} \times \ZZ[2]^{1\times m} $. 
\end{definition}

\begin{assumption}[Learning Parity with Noise Assumption]\label{assn:lpn}
The Learning Parity with Noise (LPN) assumption assumes that for  $ m = \poly(n) $ and $ t = \theta(m/\sqrt n)$, the LPN samples are indistinguishable from random. That is, for every efficient distinguisher $ \sfD $,  
\[ \big|\prob[{\vec s\gets \ZZ[2]^n \\ \vec e \gets \ZZ[2, \ham=t]^{m} }]{\sfD(\mat  A, \vec s^T \mat A + \vec e^T )= 1 } - \prob[{\vec r\gets \ZZ[2]^m}]{\sfD(\mat  A, \vec r )= 1 }\big| < \negl(n) \]

\end{assumption}

This regime of parameters $ m = \poly(n) $ and $ t = \theta(m/\sqrt n) $ is what is traditionally used to construct  public key encryption from the LPN assumption. Next, we consider the LPN problem with preprocessing: in this variant of the problem, an inefficient algorithm $ \preprocess $ is allowed to process the matrix $ \mat A $ arbitrarily to construct a ``trapdoor''. Then the distinguisher is asked to distinguish LPN sample $ (\mat A, \vec s ^T \mat A + \vec e) $ from random. The assumption states that this is difficult for higher error rates.  

\begin{definition}[LPN with Preprocessing Problem (LPNP)]\label{def:lpnp}
We say that a pair of algorithms  $ (\preprocess, \sfD) $ where $ \preprocess $ is possibly inefficient and $ \sfD $ is efficient,  solves $ \mathsf{LPN}_{n,m,t} $ if $ \sfD $ can distinguish an LPN sample from a random sample given the trapdoor $ \state $ generated by $ \preprocess(\mat A) $.
\end{definition}

The Learning Parity with Noise problem is hard even with preprocessing in the constant noise regime. 

\begin{assumption}[LPN with Preprocessing (LPNP)] \label{assn:lpnp}
Let $ m = \poly(n) $ and $ t = r \cdot m $ for any constant $ r $. For every pair of algorithms $ (\preprocess, \sfD) $ with a possibly inefficient algorithm $ \preprocess $ and efficient $ \sfD $, the following experiment is performed: Sample $ \mat A \gets \ZZ[2]^{n\times m} $ and get $ \state \gets \preprocess(\mat A) $. Then, the distinguisher $ \sfD $ given $ \state $ cannot distinguish the LPN samples from random. That is for large enough $ n $,
	\[ \Big|\prob[{\vec s\gets \ZZ[2]^n \\ \vec e \gets \ZZ[2, \ham=t]^{m} }]{\sfD(\state, \mat  A, \vec s^T \mat A + \vec e^T )= 1 } - \prob[{\vec r\gets \ZZ[2]^m}]{\sfD(\state, \mat  A, \vec r )= 1 }\Big| < \negl(n) \]
where the probability is over the code $ \mat A $, $ \vec s, \vec e , \vec r $ and the randomness of the distinguisher $ \sfD $. 
\end{assumption}

\paragraph{Discussion.}
The most important parameter of the LPN problem is its error rate, that is $r =  t/m $. The higher the error rate, the more difficult the problem. 
There are two important regimes of the error rate: $ r $ is a constant and $ r = o(\frac1{\sqrt{n}}) $. When the error rate is a constant, the hardness of LPN in this regime implies one-way functions and hence symmetric key cryptography. We do not know how to base public key encryption on error rates in this regime. When the error rate decreases below $ O(\frac1{\sqrt n}) $, we can construct public key encryption from this problem. For error rates below $ \log n/n $, the problem becomes easy. The best known algorithms for solving LPN are due to Blum Kalai and Wasserman \citep{BKW03} which solves LPN in time $ 2^{O(n/\log n)} $ requiring $ 2^{O(n/\log n)} $ samples; and Lyubashevsky \citep{L05} which solves LPN in time $ 2^{O(n/\log\log n)} $ with polynomially many samples. For structured LPN samples, more efficient algorithms are known \cite{AG11}. Our error distributions are not structured. 

Note that the lesser used variant of LPN is used here, in that we insist that the Hamming weight of the error vector is exactly $ t $ instead of a random variable. This is equivalent to the standard formulation \citep{jkpt12}.\footnote{In the search version of the problem where the adversary has to find $ \vec s $ given $ \mat A, \vec s^T \mat A + \vec e^T  $, these two versions are equivalent as $ t $ takes polynomially many values, hence we can go over all polynomially-many and try solving each exact version).} This is done for convenience and the example can be translated to the definition of LPN where the error vector is drawn from a product distribution. 

We also consider a the preprocessing variant of Learning Parity with Noise. In this variant, the adversary is allowed to preprocess the code and generate a small ``trapdoor'' to the code. Then an efficient adversary is tasked with distinguishing the LPN samples from random. The preprocessing variant of LPN assumption states that even this is hard in the constant error regime, that is when $ t/m $ is a constant. It is known that decoding linear codes is NP-hard in the worst case \citep{BruckN90}. The search analog of LPN is precisely the average-case variant of this question and is conjectured to be hard in the regime of constant noise rate. 


\paragraph{Trapdoor for Efficient Decoding.} In the public key regime, we want to show that trapdoors exist that enable effient distinguishing of LPN samples. We state the result next: that there is a way to sample a random matrix $ \mat H $ that is indistinguishable from a random matrix such that it has a trapdoor that enables efficient distinguishing. 

\begin{lemma}[Computational Trapdoor Sampling] \label{lem:lpntrapdoor}
Consider the following algorithm $ \trapsamp $ such that, $ \trapsamp $ on input $ (n, t) $ with $ t = \theta(\sqrt n) $ does the following: \\  
\noindent
\begin{minipage}{0.9\textwidth}
	\medskip
	\underline{$ \trapsamp(n, t)$}:
	\medskip
	\begin{enumerate}[noitemsep,nolistsep]
		\item Sample $ \mat A \gets \ZZ[2]^{n \times m} $, $ \mat S \gets \ZZ[2]^{n\times n} $ and $ \mat E  \in \ZZ[2]^{n \times m}$ where $ \vec e_{(i, \cdot)} \gets \ZZ[2, \ham = t]^m $. 
		\item Output $	\mat H = \matstack{\mat A \\\mat S\mat A + \mat E } , \mat E$
	\end{enumerate}
	\medskip
\end{minipage}	

\noindent
The algorithm has the following properties: 
	\begin{enumerate}[noitemsep]
		\item $ \leftspan(\mat E) \subset \leftspan(\mat H)$ where $ \leftspan(\mat T)  = \set{\vec s \mat T : \vec s \in \ZZ[2]^n}$. 
		\item The matrix $ \mat H $ is computationally indistinguishable from uniformly random matries. That is, 
		$$ \set{ \mat H \gets \trapsamp(n,t)} \approx_c \set {\mat U \gets \ZZ[2]^{2n \times 8n} }$$
		\item With overwhelming probability over the randomness of the algorithm, it outputs $ \mat E $ such that every column of $ \mat E $ has Hamming weight at most $ t $  and every row of $ \mat E  $ has Hamming weight exactly $ t $. 
	\end{enumerate}
\end{lemma}
The notion of Trapdoor sampling is very widely used in the context of learning with errors assumption. A trapdoor sampling algorithm samples along with the public matrix $ \mat A $ which is statistically close to a random matrix (representing the code/lattice), a secret ``trapdoor''. This trapdoor enables solving the bounded distance decoding problem, that is given a point close to a codeword in the code, finds the close codeword. As we know, without this trapdoor, this problem is conjectured to be hard. But the trapdoor enables solving this problem.  	

We have a computational analog of that property for LPN in the ``public-key'' regime of parameters. We construct that below. Because $ \mat E  $ is a sparse matrix, it can be used to solve the problem of distinguishing LPN samples from random and decoding noisy codewords.
\begin{proof}

By definition,  $ \leftspan(\mat E) \subset \leftspan(\mat H)$ and that each row of $ \mat E $ has Hamming weight exactly $ t $. We need to show that $ \mat H $ is indistinguishable from random and that every column of $\mat E $ has at most $ t $ ones. The former follows from the Learning Parity with Noise combined with a hybrid argument and the latter from a Chernoff bound. 

\begin{claim}
The output distribution of $ \mat H $ is computationally indistinguishable from uniform. That is, 
$$ \set{ \mat H \gets \trapsamp(n, t)} \approx_c \set {\mat U \gets \ZZ[2]^{2n \times 8n} }$$
\end{claim}
\begin{proof}
Observe that the LPN assumption can be restated as, 
The LPN assumption assumes that the following two distributions are indistinguishable: 
$$ 
\matstack{\mat A \\\vec s^T \mat A + \vec e^T } \approx_c \set{\mat U : \mat U \gets \ZZ[2]^{(n+1) \times m}}
$$ 
where $ \vec{s} \gets \ZZ[2]^n $, $ \mat A \gets \ZZ[2]^{n \times m} $, $ \vec{e} \in \ZZ[2]^m $ is a random vector of Hamming weight $ t $.
The claim then follows by applying a hybrid argument to each of the rows of $ \mat S \mat A + \mat E $ and replacing them by random vectors, by viewing them as $ \vec s_{(i,\cdot)} \mat A + \vec e_{(i, \cdot)}$ where $  \vec s_{(i, \cdot)}, \vec e_{(i, \cdot)} $ denote the $ i $-th row of matrix $ \mat S $ and $ \mat E $ and using the LPN assumption.  
\end{proof}

\begin{claim} Let $ \vec e_{(\cdot, j)} $ denote the $ j $-th column of matrix $ \mat E $. Then, 
\[ \prob[\mat E \gets \trapsamp(n,t)]{\exists j, \text{ such that, } \|{ \vec e_{(\cdot, j)}}\|_{\ham} > t } < 8n \cdot e^{-\frac{7t}{24}} \ . \]
\end{claim}
\begin{proof}
The proof follows from Chernoff bound and a union bound. For any fixed column $ j $, each coordinate $ e_{i,j} = 1 $ independently with probability $ t/8n $ where the probability is over $ i $. Hence, for any column $ j $, the expected Hamming weight is $ \frac{t}{8n} \cdot n = \frac{t}{8} $. By a Chernoff bound, we can observe the following:  
\begin{align*}
\prob[\mat E \gets \trapsamp(n,t) ]{\sum_i e_{i,j} > t } &= \prob[\mat E \gets \trapsamp(n,t) ]{\sum_i e_{i,j} > (1+7)\cdot \EE(\sum_i e_{i,j})} \\ 
&\leq e^{-7\cdot\frac{t}{8}\cdot\frac13}
\end{align*} 
where the inequality follows from the Chernoff bound in the following form: Let $ X_1, X_2, \dots X_n $ be independent random variables taking values in $ \zo $. Let $ X $ be their sum and $ \mu = \EE X $. For any $ \delta \geq 1 $, 
\[ \prob{X \geq (1+\delta)\mu} \leq e^{-\frac{\delta \mu}{3}} \ .\]
A union bound over all $ j $ gives us the required bound. 
\end{proof}
Because $ t = \sqrt n $, the failure probability is negligible. 
\end{proof}
  
\subsection{No Efficient Robust Classifier Exists}
Next, we describe a learning task where while it is possible to inefficiently perform robust classification, no efficient robust classifier exists. 
\begin{theorem}
	For an $ n $, let $ m = 8n,  t = 2n-1, \eps = 2n$. Consider the following learning task. Let $ \mat A \gets \ZZ[2]^{m\times n} $. Define $ D_0, D_1 $ as: 
	$$ D^{(\mat A)}_0 = \set{\vec s^T \mat A : \vec s \gets \zo^n} \text{ and, } D^{(\mat A)}_1 =\set{\vec s^T \mat A + \vec 1 : \vec s \gets \zo^n} \ .$$ 
	The learning task has the following properties. 
	\begin{enumerate}
		\item (Learnability) A classifier to distinguish $ D_0 $ from $ D_1 $ can be learned from the samples efficiently. Furthermore, it is easy to learn a generator/ evaluator for these distributions.  
		\item (No Efficient Robust Classifier Exists) There exists a perturbation algorithm $ \sfP $ such that there exists no efficient robust classifier $ \sfR $ such that, \[ \prob[]{\sfR(\tilde y) \in \sfR^{-1}(b)} \geq 0.5 + \negl(n) \]
		where the perturbed sample $ \tilde y $ is generated by sampling $y \gets D_b $ for a random  $ b $ and is then perturbing $ \tilde y \gets \sfP^{D_0, D_1}(y)$ such that $ \| y - \tilde y \| \leq \eps $. 
	\end{enumerate}
\end{theorem}
\begin{proof}
Learnability of this task is trivial. Given enough samples, the entire subspace spanned by $ \mat A $ is learned and can be sampled from. 

In order to show that no efficient robust classifier exists for $ \eps = 2n $, we rely on the difficulty of LPN with Preprocessing (Assumption \ref{assn:lpnp}). Consider the following perturbing adversary $ \sfP $: 
\[ \sfP(\vec y) : \text{ Output $ \tilde{\vec y} = \vec y + \vec e $ where $ \vec e \gets \ZZ[2, \ham = t]^{m} $} \]
Consider the following pair of algorithms $ \preprocess, \sfD $: $ \preprocess(\mat A) $ inefficiently finds the best possible efficient robust classifier $ \sfR $ and returns that as the trapdoor $ \state = \sfR $. The distinguisher $ \sfD $ simply runs the robust classifier $ \sfR $ and returns the answer. It can do this in polynomial time because $ \sfR $ is also polynomial time computable. 

The LPNP assumption implies that for this pair of algorithms $ (\preprocess, \sfD) $, LPN is hard to solve. That is, for $ \mat A \gets \ZZ[2]^{n\times m}, \sfR\gets \preprocess(\mat A)  $,
\begin{equation}\label{eqn:robust:lpn}
\Big|\prob[{\vec s\gets \ZZ[2]^n \\ \vec e \gets \ZZ[2, \ham=t]^{m} }]{\sfR(\mat  A, \vec s^T \mat A + \vec e^T )= 1 } - \prob[{\vec r\gets \ZZ[2]^m}]{\sfR(\mat  A, \vec r )= 1 }\Big| < \negl(n)
\end{equation} 
Now a hybrid argument finishes the proof as the following distributions are computationally indistinguishable for $ \sfR $: 
\[ (\mat A,  \vec s^T \mat A + \vec e^T )  \approx_c  (\mat A, \vec r) \equiv (\mat A, \vec r + \vec 1) \approx_c (\mat A, \vec s^T \mat A + \vec e^T + \vec 1) \]
where the two $ \approx_c $ statements follow  from \cref{eqn:robust:lpn} and the $ \equiv $ follows from the fact that adding any fixed vector to the uniform distribution still remains uniform. 

This completes the argument. 
\end{proof}

\subsection{Efficient Robust Classifier Exists but is Hard to Learn}
In this section, we describe a learning task where a robust classifier exists, but it is hard to learn. 
Consider the following classification task : Given a matrix $ \mat H \in \ZZ[2]^{2n\times 8n}$, define $ D_0, D_1 $ as: 
$$ D^{(\mat H)}_0 = \set{\vec y \in \ZZ[2]^{8n}: \mat H \vec y = 0 \bmod 2} \text{ and, } D^{(\mat H)}_1 = \set{\vec y + \vec 1 \in \ZZ[2]^{8n}: \mat H \vec y = 0 \bmod 2} $$ 
where both are uniform distributions on the sets and $ \bm{1} $ is the all ones vector on $ 8n $ dimensions.

\begin{theorem}
For an $ n $, let $ t = 2\floor{{\sqrt{n}}/{6}} -1 $, such that $ t $ is odd. Consider the following learning task. Let $ (\mat H, \mat E) \gets \trapsamp(n, t) $. Given a matrix $ \mat H \in \ZZ[2]^{2n\times 8n}$, define $ D_0, D_1 $ as: 
$$ D^{(\mat H)}_0 = \set{\vec y \in \ZZ[2]^{8n}: \mat H \vec y = 0 \bmod 2} \text{ and, } D^{(\mat H)}_1 = \set{\vec y + \vec 1 \in \ZZ[2]^{8n}: \mat H \vec y = 0 \bmod 2} $$ 
The learning task has the following properties. 
	\begin{enumerate}
		\item (Learnability) A classifier to distinguish $ D_0 $ from $ D_1 $ can be learned from the samples efficiently. Furthermore, it is easy to learn a generator/ evaluator for these distributions.  
		\item (Existence of an Efficient Robust Classifier) There exists an efficient robust classifier $ \sfR $ such that, \[ \prob[y\gets D_b]{B(y,\eps) \in \sfR^{-1}(b)} \geq 0.99 \]
where $ \eps  = \floor{\sqrt n}$ and $ B(y, \eps)  = \{y' : \|y-y'\|_{\sf{Ham}} \leq \eps \}$. 
		\item (Unlearnability of Robust Classifier) There exists a perturbation algorithm such that no efficiently learned classifier can classify better than chance. 
	\end{enumerate}
\end{theorem}
We drop $ \mat H $ from the notation to avoid clutter and denote the distributions as $ D_0, D_1 $. Here $ \mat H $ functions as the parity check matrix of the code $ D_0 $ and $ D_1 $ is a shift of the code. Observe that Part (1): distinguishing between $ D_0 $ and $ D_1 $ is easily done by Gaussian elimination. 

We want to show that (2) a robust classifier exists, and, (3) it is difficult to find any robust classifier efficiently. We argue this in the subsequent claims. 

\begin{lemma} [Existence of Robust Classifier] 
	Consider the following robust classifier: 
	
	\begin{minipage}{0.6\textwidth}
		\medskip
		\underline{\sf{Robust Classifer} $ \sfR_{\mat E}(\tilde{\vec{y}}) $}: 
		\begin{enumerate}[noitemsep,nolistsep]
			\item Compute $ \vec z = \mat E\tilde{\vec y} \bmod 2$. 
			\item If $ \norm{\vec z}_{\sf{Ham}} \leq n $ output $ 0 $ otherwise, output $ 1 $. 
		\end{enumerate}
		\medskip
	\end{minipage}

\noindent
Then, the following holds: 	$$ \prob[\vec y\gets D_b]{\sfR(\tilde{\vec y}) = b \text{ for all $ \tilde{\vec y} \in B(\vec y, \eps) $}} \geq 0.99$$ 
	for $ \eps = \floor{\sqrt{n} }$ and $ B(\vec y, \eps) = \set{\vec y : \norm{\tilde{\vec y}- \vec y}_{\sf{Ham}} \leq \eps } $. 
\end{lemma}

\begin{proof}
	
	\medspace
	The correctness of the robust classifier follows from the fact that $ \mat E $ is a sparse matrix where each column has Hamming weight at most $ t $. Consider the case when $\vec y \gets D_0$, the other case is analogous. Observe that, 
	$$\tilde{\vec y} = \vec y + \vec \epsilon \bmod 2$$ 
	where $ \norm{\vec \epsilon}_{\sf{Ham}} \leq \eps \leq  \sqrt{n}$. Hence,
	$$\mat E \tilde{\vec y}   \bmod 2 = \mat E (\vec y+ \vec \epsilon) = \mat E \vec \epsilon \pmod 2$$
	where the second equality follows from the fact that $ \mat H y = 0 \bmod 2 $ and that $ \leftspan(E) \subseteq \leftspan(H) $. Observe that each column of $ \mat E $ has at most $ t $ ones and that the Hamming weight of $ \vec \epsilon $ is at most $ \eps $. As, $ \mat E \vec \epsilon = \sum_{j: \epsilon_j = 1} \vec e_{(\cdot, j)}$, we can bound the Hamming weight $ \norm{\mat E \vec \epsilon}_{\sf{Ham}} \leq t \norm{\vec\epsilon}_{\sf{Ham}} \leq t \cdot \eps \leq n/3 $. 
	Hence the classifier would always correctly classify adversarially perturbed samples from $ D_0 $.  
	
	In the other case when $ b = 1 $ observe that $ \mat E\cdot \vec 1 = \vec 1 $ because each row of $ \mat E $ has Hamming weight $ t $ which is odd. Hence the Hamming weight of $ \vec z $ is at least $2n - n/3 > n$ in this case and would be classified correctly. This proves that a robust classifier exists. 
\end{proof}

\begin{lemma}[Hardness of Learning a Robust Classifier] 
	There exists a perturbation algorithm $ \sfP $ such that for every polynomial time learner $ \sfL $, the learner $ \sfL $ has no advantage over chance in classifying examples perturbed by $ \sfP $.  That is, 
	$$
	\bigprob[\mat H, \mat E \gets \trapsamp(n,t); \\    \vec y \gets D^{(\mat H)}_b \text{ where $ b\gets \zo$} \\ \tilde{\vec y} \gets \sfP^{D_0, D_1}( \vec y); \\ b' \gets \sfL^{D_0, D_1}(\tilde{\vec y}) ]{b  = b'} \leq \frac12 + \negl(n)
	$$
\end{lemma}

\begin{proof}
	This proof is identical to the proof of security of Aleknovich's public key encryption scheme \cite{Ale03}. 
	
	Observe that $ D_0, D_1 $ are completely specified by the matrix $ \mat H $. So, the learner gets $ \mat H $ instead of sample access. Consider the following random perturbation algorithm $ \sfP $:  $$ \sfP(\vec y): \text{Output } \tilde{\vec y} = \vec x + \vec \epsilon \text{, where } \vec \epsilon \gets \ZZ[2, \sf{Ham} = t]^{m}$$
	where  $ \ZZ[2, \sf{Ham} = t]^{m} $ is the distribution on vectors of Hamming weight $ t $.  This adversary is adding allowable amount of error as $ t < \eps = \sqrt n $. 
	
	Suppose an efficient learner $ \sfL $ exists that can succeed in this game with high probability, we can break the learning parity with noise assumption. This is done in two steps. In the first step, we replace the parity check matrix $ \mat H $ with a uniformly random matrix $ \mat H' $ this should not noticeably change the success probability because the two distributions are indistinguishable. In the second step, now observe that $ \mat H' $ is a uniformly random parity check matrix hence gives rise to a random code. Now we can apply the LPN assumption again, this time to replace the error $ \vec \epsilon $ by a uniformly random vector and not noticably change the success probability. This is a contradiction.   

\end{proof}
}


\section{Learning with Errors}\label{sec:lwe}

\subsection{Preliminaries}

In this section, we define the learning with errors problem and describe the notion of trapdoor sampling that it supports. In this section, the norm used is the $ \ell_\infty $ norm obtained by embedding $ \Zq $ in $ \ZZ[] $. That is, for vectors $ \vec x, \vec y\in \Zq^n $, $ \|\vec x - \vec y\| = \max_i |x_i - y_i| $ where $ |z| $ for $ z\in \Zq $ is obtained by embedding $ z \in \set{-\floor{q/2},\dots , -1, 0, 1, \dots \floor{q/2} } $ and taking the absolute value. 

\begin{definition}[Learning with Errors Problem]
For $ n,m \in \NN $ and modulus $ q \geq 1 $, distribution for error vectors $\chi \subset \Zq $, a Learning with Errors (LWE) sample is obtained  by sampling $ \vec{s} \gets \Zq^n $, $ \mat A \gets \Zq^{n \times m} $, $ \vec{e} \gets \chi^m $ and outputting $ (\mat A, \vec s^T \mat A + \vec e^T \bmod q) $.  

We say that an algorithm solves $ \mathsf{LWE}_{n,m,q, \chi} $ if it distinguishes LWE sample from a random sample distributed as $\Zq^{n\times m} \times \Zq^{1\times m} $. 
\end{definition}

\begin{assumption}[Learning with Errors Assumption]
	The Learning with Errors (LWE) assumption assumes that for  $ m = \poly(n) $, $ q = \Omega(n^3)$ and $ \chi $ is truncated discrete gaussian over $ \Zq $ with standard deviation $ q/n^2$ truncated to $ q/2n $, the LWE samples are indistinguishable from random. That is, for every efficient distinguisher $ \sfD $,  
	\[ \big|\prob[{\vec s\gets \ZZ[2]^n \\ \vec e \gets \ZZ[2, \ham=t]^{m} }]{\sfD(\mat  A, \vec s^T \mat A + \vec e^T )= 1 } - \prob[{\vec r\gets \ZZ[2]^m}]{\sfD(\mat  A, \vec r )= 1 }\big| < \negl(n) \]
\end{assumption}
We have written specific versions of the LWE assumption. LWE is conjectured to be hard for a large setting of parameters. For a discussion on parameters, see \cite{Pei16}. 

\begin{definition}[LWE with Preprocessing Problem (LWEP)]
We say that a pair of algorithms  $ (\preprocess, \sfD) $ where $ \preprocess $ is possibly inefficient and $ \sfD $ is efficient,  solves $ \mathsf{LWE}_{n,m,t} $ if $ \sfD $ can distinguish an LWE sample from a random sample given the trapdoor $ \state $ generated by $ \preprocess(\mat A) $.
\end{definition}

The Learning Parity with Noise problem is hard even with preprocessing in the constant noise regime. We state the assumption below formally. 

\begin{assumption}[LWE with Preprocessing (LWEP)] Let $ m = n\log q+ 2n, q= n^3 $  and $\chi$ is a discrete Gaussian with standard deviation $q/100$ truncated to $ q/10 $. For every pair of algorithms $ (\preprocess, \sfD) $ with a possibly inefficient algorithm $ \preprocess $ and polynomial time $ \sfD $, the following experiment is performed: Sample $ \mat A \gets \ZZ[2]^{n\times m} $ and get $ \state \gets \preprocess(\mat A) $. Then, the distinguisher $ \sfD $ given $ \state $ cannot distinguish the LPN samples from random. That is,  
	\[ \Big|\prob[{\vec s\gets \ZZ[2]^n \\ \vec e \gets \ZZ[2, \ham=t]^{m} }]{\sfD(\state, \mat  A, \vec s^T \mat A + \vec e^T )= 1 } - \prob[{\vec r\gets \ZZ[2]^m}]{\sfD(\state, \mat  A, \vec r )= 1 }\Big| < \negl(n) \]
\end{assumption}

\begin{definition}[Lattice Trapdoor] For a matrix $ A \in \Zq^{n \times m} $, we denote by $ L^\perp $ the \textdef{dual lattice} of $ A $ composed of all vectors in the kernel of $ A $: 
 $$ L^\perp = \set{x \in \ZZ^m : Ax = 0 \bmod q} $$ 
 A \textdef{trapdoor} for $ A $ is a short basis for the lattice $ L^\perp(A) $.
\end{definition}
In the case of LWE, it is known that we can sample matrices $ \mat A $ from a distribution statistically close to uniformly random along with a trapdoor which allows for efficient distinguishing and recovering the lattice point from a noisy one, for close distances (this is referred to as bounded distance decoding).

\begin{theorem}[Trapdoor Sampling \citep{GPV08}] There exists an algorithm $ \trapsamp $ such that, $ \trapsamp $ on input $ (q, m, n) $ where $ m \geq n\log q + 2n $ outputs a pair of matrices $ (\mat A, \mat T) $ where $ \mat A\in \Zq^{n\times m} $, $ \mat T \in \Zq^{m\times n\log q}$, with the following properties: 
	\begin{itemize}[noitemsep]
		\item $ \mat A \mat T = \mat 0 \mod q $.
		\item The output distribution of $ \mat A $ is statistically close to uniform (total variation distance $ < 2^{-O(n)} $).
		\item $ \mat T $ has only zero-one entries.  
	\end{itemize}
\end{theorem}

\subsection{No Efficient Robust Classifier Exists} 

In this section we describe a learning task based on LWE  that has no robust classifier. This is identical to the LPN based task except the noise distribution is set differently.

\begin{theorem}
	For any $ q = n^3$ and $m= n\log q + 2n $, and $ \chi $ is a discrete Gaussian with standard deviation $q/100$ truncated to $ q/10 $.  Consider the following learning task. Let $ \mat A \gets \ZZ[q]^{m\times n} $. Define $ D_0, D_1 $ as: 
	$$ D^{(\mat A)}_0 = \set{\vec s^T \mat A : \vec s \gets \zo^n } \text{ and, } D^{(\mat A)}_1 =\set{\vec s^T \mat A + \vec{\frac{q}2} : \vec s \gets \zo^n} \ .$$ 
	The learning task has the following properties. 
	\begin{enumerate}
		\item (Learnability) A classifier to distinguish $ D_0 $ from $ D_1 $ can be learned from the samples efficiently. Furthermore, it is easy to learn a generator/ evaluator for these distributions.  
		\item (No Efficient Robust Classifier Exists) There exists a perturbation algorithm $ \sfP $ such that there exists no efficient robust classifier $ \sfR $ such that, \[ \prob[]{\sfR(\tilde y) \in \sfR^{-1}(b)} \geq 0.5 + \negl(n) \]
		where the perturbed sample $ \tilde y $ is generated by sampling $x \gets D_b $ for a random  $b $ and is then perturbing $ \tilde x \gets \sfP^{D_0, D_1}(x)$ such that $ \| y - \tilde y \| \leq q/10 $. 
	\end{enumerate}
\end{theorem}
The proof is identical to the LPN case, with the perturbation adversary $ \sfP $ instead adding noise distributed according to $ \chi^m $.

\subsection{An Efficient Robust Classifier Exists but is Hard to Learn} 

We define the classification task ($ D_0, D_1 $) as follows: 
Given a matrix $ \mat A \in\Zq^{n \times m} $ consider distributions $ D_0 $ and $ D_1 $ defined as: 
$$ D^{(\mat A)}_0 = \set{\vec s^T \mat A : \vec s \in \Zq^n} \text{ and, } D^{(\mat A)}_1 = \set{\vec s^T \mat A + \frac q2 \cdot \vec{1}^T: \vec s \in \Zq^n} \ .$$ 
where both are uniform distributions on the sets and $ \bm{1} $ is the all ones vector on $ m $ dimensions. We drop $ \mat A $ from the notation to avoid clutter and denote the distributions as $ D_0, D_1 $. 

Hence, the task consists of distinguishing lattice vectors from an affine shift of the lattice. That is,  given a vector $ x \in (D_0 \cup D_1) $, classify weather $ x\in D_0 $ or $ x \in D_1 $. Gaussian elimination accomplishes this task easily. We want to show that (a) a robust classifier exists, and, (b) it is difficult to find any robust classifier efficiently. We argue this based on the learning with errors assumption. 

At the heart of the construction is the idea of lattice trapdoors. For a matrix $ \mat A \in \Zq^{n\times m} $, the trapdoor is a ``short'' matrix $ \mat T  $ such that $ \mat A \mat T = \vec 0 \mod q $. There are two key properties of these trapdoors that we leverage: 
(1) This short matrix allows us to solve the ``bounded distance decoding (BDD)'' problem : that is, given a vector close to the lattice, find the closest lattice vector efficiently. Hence, the trapdoor functions as a robust classifier. Also, we can efficiently sample a random matrix $ \mat A $ together with such a trapdoor.  
(2) It is hard to find such a trapdoor given the matrix $ \mat A $, even when it exists, because these trapdoors allow us to solve the Learning with Errors problem. This allows us to show that the robust classifier is hard to learn. 

\begin{theorem}
	For any $ q = n^3 $ and  $ m =  n \log q + 2n $, consider the following learning task. Let $ (\mat A, \mat T) \gets \trapsamp(n, m, q) $. Given a matrix $ \mat A \in \Zq^{n\times m}$, define $ D_0, D_1 $ as: 
$$ D^{(\mat A)}_0 = \set{\vec s^T \mat A : \vec s \in \Zq^n} \text{ and, } D^{(\mat A)}_1 = \set{\vec s^T \mat A + \frac q2 \cdot \vec{1}^T: \vec s \in \Zq^n} \ .$$ 
	The learning task has the following properties. 
	\begin{enumerate}
		\item (Learnability) A classifier to distinguish $ D_0 $ from $ D_1 $ can be learned efficiently. 
		\item (Existence of Robust Classifier) There exists a robust classifier $ \sfR $ such that, \[ \prob[y\gets D_b]{B(y,q/4m) \in \sfR^{-1}(b)} \geq 0.99 \]
		where  $ B(y, \eps)  = \{y' \in \Zq^m : \|y-y'\|_{\infty} \leq \eps  \}$. 
		\item (Unlearnability of Robust Classifier) There exists a perturbation algorithm such that no efficiently learned classifier can classify better than chance. 
	\end{enumerate}
\end{theorem}

\begin{lemma} [Existence of a Robust Classifier] Consider the following robust classifier $ \sfR $:
	
\begin{minipage}{0.6\textwidth}
\medskip
\underline{\sf{Robust Classifer} $ \sfR_{\mat T}(\tilde{\vec{y}}) $}: 
\begin{enumerate}[noitemsep,nolistsep]
	\item Compute $ \vec z = \tilde{\vec y}^T \mat T  \bmod q$. 
\item If $ \vec z \in \set{\frac{-q}4, \dots, \frac q4}^n $ output $ 0 $ otherwise, output $ 1 $. 
\end{enumerate}
\medskip
\end{minipage}	
	
\noindent Then, 
	$$ \prob[\vec x\gets D_b]{\sfR(\tilde{\vec x}) = b \text{ for all $ \tilde{\vec x} \in B(\vec x, q/4m) $}} \geq 0.99$$ 
\end{lemma}
\begin{proof}
The correctness of the robust classifier follows from the fact that $ \mat T $ is a zero-one matrix and that the errors are bounded in size. Consider the case when $\vec y \gets D_0$, the other case is analogous. Observe that, 
$$\tilde{\vec y} = \vec y + \vec e \bmod q = \vec s^T \mat A + \vec e^T  \bmod q $$ 
where $ \|\vec e\|_\infty \leq  \frac{q}{4m}$. Hence,
$$\tilde{\vec y}^T \mat T \bmod q = (\vec s^T \mat A + \vec e^T) \mat T \bmod q = \vec e^T \mat T \bmod q$$ 
As $ \mat T $ has only zero-one entries, $ \vec e^T \mat T $ is bounded over integers with the absolute value of each coordinate being at most $ m \cdot \| \vec e\|_\infty \leq \frac q4 $. 
This implies that the robust classifier would correctly output 0 when given perturbed samples from $ D_0 $. 
\end{proof}

In order to show that it is difficult to recover the robust classifier, we rely on the learning with errors assumption. We consider a perturbation adversary that simply adds random noise to the sample it receives.
\begin{lemma}[Hardness of Learning a Robust Classifier] 
There exists a perturbation algorithm $ \sfP $ such that for every polynomial time learner $ \sfL $, the learner $ \sfL $ has no advantage over chance in classifying examples perturbed by $ \sfP $.  That is, 
$$
\bigprob[\mat A, \mat T \gets \trapsamp(n,m,p); \\    \vec x \gets D^{(\mat A)}_b \text{ where $ b\gets \zo$} \\ \tilde{\vec x} \gets \sfP^{D_0, D_1}(x); \\ b' \gets \sfL^{D_0, D_1}(\tilde x)  ]{b  = b'} \leq \frac12 + \negl(n)
$$
\end{lemma}
\begin{proof}
Observe that $ D_0, D_1 $ are completely specified by the matrix $ \mat A $ and given $ \mat A $ can be sampled efficiently. So, it suffices to give the learner $ \mat A $ instead of sample access. Consider the following random perturbation algorithm $ \sfP $:  $$ \sfP(\vec x): \text{Output } \tilde{\vec x} = \vec x + \vec e \text{, where } \vec e \gets \chi^m.$$
So, the experiment above is equivalent to the following: 
$$
\bigprob[\mat A, \mat T \gets \trapsamp(n,m,p); \\  \vec s \gets \Zq^n, \vec e \gets \chi^m, b\gets \zo \\ b' \gets \sfL(\mat A,\vec s^T \mat A + \vec e^T + b\frac q2 \cdot \vec{1}^T )  ]{b  = b'}  \leq \frac12 + \negl(n)
$$
The cruical observation is that the learner's job is to distinguish LWE samples $ (\mat A, \vec s^T \mat A + \vec e^T) $ from shifted LWE samples $ (\mat A, \vec s^T \mat A + \vec e^T + \frac q2 \vec 1^T) $. The LWE assumption implies that this is difficult because the two distributions are indistinguishable. That is, 
$$
(\mat A,\vec s^T \mat A + \vec e^T) \approx_c (\mat A,\vec r^T ) \approx_c (\mat A,\vec s^T \mat A + \vec e^T + \frac q2 \cdot \vec{1}^T )
$$
and hence no efficient adversary $ \sfL $ can distinguish between the distribution when $ b= 0 $ from when $b=1 $. And hence for any efficient adversary, the success probability of classifying these perturbed instances is negligibly close to a half, as desired. 
\end{proof}

Hence, we have described a learning task that is learnable, has a robust classifier, but robust classifiers are computationally hard to learn. 
\section{Using Pseudorandom Functions and Error Correcting Codes}

In this section, we formally describe the hard-to-robustly learn task based on one-way functions. There are two main ingredients that we use to construct the learning task: Error Correcting Codes (ECCs) and Pseudorandom Functions (PRFs).

An uniquely decodable binary error correcting code allows encoding messages to redundant codewords such that from any codeword perturbed to some degree, we can recover the encoded message.

\begin{definition}[Uniquely Decodable Error Correcting Code]
An uniquely decodable binary error correcting code, $ C : \zo^n \rightarrow \zo^m $ consists of two efficient algorithms $ \encode, \decode $. The code tolerates error fraction $ e $ if for all messages $ x\in \zo^n $,  	$$ \decode(\tilde y) =  x \text{ for all $ \tilde y \in B(\encode(x), em)$}$$
where $ B(\encode(x), em) $ denotes the Hamming ball of radius $ em $.
\end{definition}

We know very good error correcting codes.
\begin{theorem}[\cite{GI01}] For any constant $ \gamma > 0 $, there exists a binary error correcting code $ C: \zo^{n} \rightarrow \zo^{m} $ where $ m=O(n/\gamma^3) $ with a decoding radius of $ (\frac14 - \gamma) m$ with polynomial time encoding and decoding.
\end{theorem}
We will use this coding scheme with $ \gamma = 1/8 $ giving us an error correcting code $ C: \zo^n \rightarrow \zo^m $ where $ m = \theta(n) $ and tolerates $ m/8 $ errors for unique decoding.

A \textdef{pseudorandom function} is a keyed function $ F_k : \zo^{n-1} \rightarrow \zo$ where the secret key is picked uniformly random such that, for every efficient adversary, the output of the function is indistinguishable from the output of a random function.  A more formal definition is given below. It is known that pseudorandom functions can be constructed from one-way functions.
\begin{definition}[\cite{GGM86}]
A family of polynomial-time computable functions $\cF= \set{\cF_n}$ where  $\cF_n = \set{F_{k} : \zo^n \rightarrow \zo } $ where $ k \in \zo^n $ and $ n \in \NN $ is \textdef{pseudorandom} if every polynomial time computable adversary $ \sfA $ cannot distinguish between $ \cF $ and uniformly random function. That is,
\[ \abs{\prob[k\gets \zo^n]{\sfA^{F_k}(1^n) = 1 } - \prob[U_n \gets \cU_n]{\sfA^{U_n}(1^n) = 1 }} < \negl(n)\]
where $ \cU_n $ is the uniform distribution over all functions from $ \zo^n $ to $ \zo $.
\end{definition}
\begin{theorem}[\cite{GGM86}] Pseudorandom functions exist if one-way functions exist.
\end{theorem}

Next, we informally describe the learning task.
Consider the following learning task: The two distributions $ D_0, D_1 $ are parameterized by the PRF key $ k $ and defined as follows:
\[ D_0 = (0, \encode(x, F_k(x))) \text{ and, } D_1 = (1, \encode(x, 1- {F_k(x)})) \ .\]
So, the two distributions are tuples where the first half is which distribution the sample was taken from and the second an error correcting code applied to the tuple $ (x, F_k(x) + b) $, that is, either the PRF evaluation at the location $ x $ or its complement.
Note that without the first bit, classifying the original distributions is computationally infeasible. The pseudorandom function looks random at every new location. Including the bit in the sample itself makes the unperturbed classification task easy. The error correcting code ensures that we have a robust classifier.

\begin{theorem}
\label{thm:prfs}
Let $ \set{F_k} $ for $ F_k : \zo^{n-1} \rightarrow \zo $  be a pseudorandom function family and $ C : \zo^n \rightarrow \zo^m $ where $ m = \theta(n) $ be an efficiently decodable error correcting code  with decoding algorithm $ \sf{Decode} $ that tolerates $ m/8 $ errors.

Consider the following learning task. For a random pseudorandom function key $ k $, define:
$$ D^{(k)}_0 = \set{(0, C(x, F_k(x))) : x\gets \zo^{n-1}} \text{ and, }D^{(k)}_1 = \set{(1, C(x, 1-F_k(x))) : x\gets \zo^{n-1}} $$
supported on $ \zo^m $. The learning task has the following properties.
\begin{enumerate}
\item (Easy to Learn) A classifier to distinguish $ D_0 $ from $ D_1 $ can be learned from the samples efficiently.
\item (Robust Classifier Exists) There exists a robust classifier $ \sfR $ such that, \[ \prob[y\gets D_b]{B(y, m/8) \in \sfR^{-1}(b)} \geq 0.99 \]
where $ m/8 $ is the decoding radius and $ B(y, d)  = \{y' : \|y-y'\|_{\sf{Ham}} \leq d \}$.
\item (A Robust Classifier is hard-to-learn) There exists a perturbation algorithm such that no efficiently learned classifier can classify perturbed adversarial examples better than chance.
\end{enumerate}
\end{theorem}
\begin{proof}
To prove Part (1) consider the classifier that outputs the first bit. It works correctly on instances from the distributions.
To prove Part (2), we rely on the decoding algorithm. After $ d = m/8 $ edits to the sample, we can recover the underlying message by ignoring the first bit of the tuple and decoding the rest to get the underlying message of the form $ (x, c) $ and then use the PRF to classify. More formally, consider the following robust classifer:

\noindent
\begin{minipage}{0.8\textwidth}
	\medskip
	\underline{\sf{Robust Classifer} $ \sfR_{k}(\tilde{{y}}) $ where $ \tilde y \in \zo^{m+1} $}:
	\begin{enumerate}[noitemsep,nolistsep]
		\item Let $ (x, c) = \mathsf{Decode}(\tilde{y}_{2:m+1}) $ where $ \tilde{y}_{2:m+1} $ are all of $ \tilde y $ but the first bit.
		\item Output $ 0 $ if $ c = F_k(x) $ else output 1.
	\end{enumerate}
	\medskip
\end{minipage}

Observe that error correcting code ensures that from every perturbed sample, we efficiently recover the encoded message. And then because the message is of the form $ (x, F_k(x) + b) $ for class $ b $, this allows for correct classification.

To show Part (3), we rely on the unlearnability of the PRF. Consider a perturbing adversary that replaces the first bit of the sample by 0. Classification is now equivalent to predicting $ F_k(x) $ given $ x $. 
Because predicting $ F_k(x) $ is computationally infeasible to learn, so is a robust classifier.
\end{proof}

Note that, compared to the previous counter-examples, this example does not rely on public key assumptions. The reason for that is that the samples here are ``evasive''. In that there is no way to generate fresh samples from the two distributions. So, we cannot translate this to a public key encryption scheme because to encrypt, we need a samples from the distributions $ D_0, D_1$ along with the perturbing adversary and we do not have access to these samples.

The hardness of this task comes from the hardness of learning the PRF and not from the perturbations. This is different from the schemes based on LPN and LWE.


\section{Using Average-Case Hardness and Error Correcting Codes}
\label{sec:avgcase-ecc}
\preetum{Added.}\akshay{Looks good.}

In this section, we formally state \cref{infthm:avg-case}
and provide the proof outlined in \cref{sec:tech:avgcase}.
We also give an alternative construction that relies on
one-way-permutations, but yields a classification problem with distributions that
are efficiently samplable.

We first need the notion of an average-case hard function.
\begin{definition}[Average-Case Hard]
A boolean function $g: \{0, 1\}^n \to \{0, 1\}$ is
\emph{$(s, \delta)$-average-case hard}
if for all non-uniform probabilistic algorithms $A$ running in time $s$,
$$\Pr_{A, x\in \{0, 1\}^n}[A(x) \neq g(x)] \geq \delta$$
\end{definition}

There exists functions $g$ which are
$(2^{\Theta(n)}, 1/2 - 2^{-\Theta(n)})$-average-case hard
(a random function $g$ will suffice with constant probability).

\begin{theorem}
\label{thm:avgcase}
Let $g: \zo^{n} \rightarrow \zo $ be a function that is
$(2^{\Theta(n)}, 1/2 - 2^{-\Theta(n)})$-average-case hard,
and let $ \encode : \zo^{n+1} \rightarrow \zo^m $ where $ m = \theta(n) $
be an efficiently decodable error correcting code
with decoding algorithm $ \sf{Decode} $ that tolerates $ m/8 $ errors.

Consider the following classification task.
Define:
$$
D_0 = \set{(0, \encode(x, g(x))) : x \gets \zo^{n}}
\text{ and }
D_1 = \set{(1, \encode(x, 1-g(x))) : x \gets \zo^{n}}
$$
This classification task has the following properties.
\begin{enumerate}
\item (Easy to Classify) An efficient classifier to distinguish $ D_0 $ from $ D_1 $ exists.
\item (Robust Classifier Exists) There exists a inefficient robust classifier
$ \sfR $ such that, \[ \prob[y\gets D_b]{B(y, m/8) \in \sfR^{-1}(b)} \geq 0.99 \]
where $ m/8 $ is the decoding radius and $ B(y, d)  = \{y' : \|y-y'\|_{\sf{Ham}} \leq d \}$.
\item (No Efficient Robust Classifier Exists) There exists a perturbation
algorithm $ \sfP $ such that there exists no polynomial-time robust classifier $ \sfR $ such that,
\[ \prob[]{\sfR(\tilde y) \in \sfR^{-1}(b)} \geq 0.5 + \negl(n) \]
where the perturbed sample $ \tilde y $ is generated by sampling $y \gets D_b $ for a random  $ b $ and is then perturbing $ \tilde y \gets \sfP^{D_0, D_1}(y)$ such that $ \| y - \tilde y \| \leq \eps $.
\end{enumerate}
\end{theorem}

\begin{proof}
This proof closely follows the proof of \cref{thm:prfs}.
For Part (1), the classifier that simply outputs the first bit is always
correct.
For Part (2), we can robustly classify by using the error correcting code to
recover the message $(x, g(x))$ or $(x, 1-g(x))$, and then we can compute the
function $g$ to distinguish between these cases.
Specifically, the robust classifier is identical to the one presented in the
proof of \cref{thm:prfs}, but computing the function $g$ instead of $F_k(x)$.
For Part (3), we rely on the average-case hardness of $g$.
Consider the perturbation adversary that replaces the first bit of the sample by
0.
Now, classifying $D_0$ vs $D_1$ with non-negligible advantage is equivalent to
predicting $g(x)$ given $x$ with non-negligible advantage.
This is impossible in polynomial time by the average-case hardness of $g$,
and thus efficient robust classification is impossible.
\end{proof}

We now describe how to achieve the above properties with distributions that are
efficiently samplable.
First, recall the notion of a \emph{hard-core bit}:
Let $f: \zo^n \to \zo^n$ be a one-way function.
A predicate $b: \{0, 1\}^n \to \{0 ,1\}$ is a \emph{hard-core bit for $f$}
if for all probabilistic polynomial-time algorithms $A$,
$$\Pr_{x \gets \zo^n}[A(f(x)) = b(x)] \leq \frac{1}{2} + \negl(n)$$
The construction is as follows.

\begin{theorem}
Let $f: \zo^{n} \rightarrow \zo^n $ be a one-way permutation,
and let $b: \zo^n \to \zo$ be a hard-core bit for $f$.
Let $ \encode : \zo^{n+1} \rightarrow \zo^m $ where $ m = \theta(n) $
be an efficiently decodable error correcting code
with decoding algorithm $ \sf{Decode} $ that tolerates $ m/8 $ errors.

Consider the following classification task.
Define:
$$
D_0 = \set{(0, \encode(f(x), b(x))) : x \gets \zo^{n}}
\text{ and }
D_1 = \set{(1, \encode(f(x), 1-b(x))) : x \gets \zo^{n}}
$$
This classification task has the following properties.
\begin{enumerate}
\item (Easy to Classify) An efficient classifier to distinguish $ D_0 $ from $ D_1 $ exists.
\item (Robust Classifier Exists) There exists a inefficient robust classifier
$ \sfR $ such that, \[ \prob[y\gets D_b]{B(y, m/8) \in \sfR^{-1}(b)} \geq 0.99 \]
where $ m/8 $ is the decoding radius and $ B(y, d)  = \{y' : \|y-y'\|_{\sf{Ham}} \leq d \}$.
\item (No Efficient Robust Classifier Exists) There exists a perturbation
algorithm $ \sfP $ such that there exists no polynomial-time robust classifier $ \sfR $ such that,
\[ \prob[]{\sfR(\tilde y) \in \sfR^{-1}(b)} \geq 0.5 + \negl(n) \]
where the perturbed sample $ \tilde y $ is generated by sampling $y \gets D_b $ for a random  $ b $ and is then perturbing $ \tilde y \gets \sfP^{D_0, D_1}(y)$ such that $ \| y - \tilde y \| \leq \eps $.
\item (Efficiently Samplable) The distributions $D_0, D_1$ can be sampled in polynomial time.
\end{enumerate}
\end{theorem}
\begin{proof}
Parts (1)-(3) follow exactly as in the proof of \cref{thm:avgcase}.
Note that an inefficent distinguisher can invert $f(x)$ to find $x$, and compute
$b(x)$.
For Part (4), both distributions are clearly efficiently samplable, by first
sampling $x$ and then computing $f(x), b(x)$.
\end{proof}

\section{Cryptography from Robustly Hard Tasks}
In this section, we show that the existence of tasks with a provable gap in classification and robust classification implies one-way functions and hence a variety of cryptographic primitives that include pseudorandom functions, symmetric key encryption among others. 

\begin{theorem} Provably hard-to-learn robust classifiers imply one-way functions. 
Given a learning task $ D_0, D_1 $ such that, 
\begin{enumerate}
\item (Robust Classifier Exists) There exists a robust classifier $ \sfR $ such that, \[ \prob[y\gets D_b]{B(y,d) \in \sfR^{-1}(b)} \geq 0.90 \]
where $ d $ is the decoding radius and $ B(y, d)  = \{y' : \|y-y'\|_{\sf{Ham}} \leq d \}$. 
\item (A Robust Classifier is hard-to-learn) There exists an efficient perturbing adversary $ \sfP $ such that every efficiently learned classifier $ \sfL $ is not a robust classifier. That is, for a learning task  $ D_0,D_1 \gets \sf{Samp}(n) $ and classifier $ \sfL $, 
$$
\prob[]{\sfL^{D_0, D_1}(\tilde x) = b} \leq \frac12 + 0.1 \ .
$$
where the perturbed sample $ \tilde x \in B(x, d) $ is generated by sampling $x \gets D_b \text{ for a random }b\gets \zo $ and is then perturbing $ \tilde x \gets \sfP^{D_0, D_1}(x)$. The probability is over the entire experiment from sampling the learning tasks to the randomness of the perturbation algorithm and the classifier. 
\end{enumerate}
Then one-way functions exist. 
\end{theorem}

The proof of this theorem relies on fact that we can construct one-way functions from any two distributions that are staistically far and computationally close. 
The two distributions considered are  the perturbed distributions. That is, 
$$ D'_0 = \set{\sfP(x) : x \gets D_0} \text{ and, } D'_1 = \set{\sfP(x) : x \gets D_1}$$
We show that these two distributions are statistically far and yet computationally indistinguishable giving one-way functions. They are statistically far because the robust classfier can distinguish between them. Hence, the total variation distance between the two has to be large. And that they are computationally close because no efficient algorithm can distinguish between the two. Hence one way functions exist. 

\begin{proof}

We formally state the theorem used below.
\begin{theorem}[Folklore, see e.g., Chap.\ 3, Ex.\ 11 \cite{Gol01}] 
	Given a pair of distributions $ (X_0, X_1) \gets \cF $ over $ \cX $ that are statistically far, 
	\[ d_{TV}(X_0, X_1) = \max_{A:\cX \rightarrow [0,1]}  \EE_{x\gets X_0}\sfA(x) - \EE_{x\gets X_1}\sfA(x) > 0.8 \]
	and computationally indistinguishable. That is for every polynomial time adversary $ \sfA $ that gets sample access to the distributions, 
	
	\[  \EE_{x\gets X_0}\sfA^{(X_0, X_1)}(x) - \EE_{x\gets X_1}\sfA^{(X_0, X_1)}(x) < 0.4 \]
	Then one-way functions exist.\footnote{The constants in the equations are fairly arbitrary. We can replace them by any constants $ \alpha, \beta $ where $ \alpha^2 > \beta $ and the result holds.}
\end{theorem}

We want to show that these two distributions are statiscally far and computationally close. This relies on the existence of the robust classifier and the difficultly of learning one respectively. 

We start by showing that, $ d_{TV}(D'_0,  D'_1) \geq 0.8 $. To observe this, consider the robust classifier as the distinguisher. This implies that,  
\[ d_{TV} \geq \EE_{x\gets D'_1}[R(x)] - \EE_{x\gets D'_0}[R(x)] \geq 0.9 - 0.1 \geq 0.8 \]

On the other hand, any efficient distinguisher cannot distinguish between the samples by the assumption. Hence we are done. 
\end{proof}

Another reasonable definition, from which we don't know one-way functions is the following: there exists a perturbation adversary $ \sfP $ that given oracle access to the underlying classifier finds counter examples. That is, $ \prob[x\gets D_b]{\sfR(\sfP^{\sfR, D_0, D_1}(x)) \neq b} \geq 0.4 $. For this definition, using standard min-max arguments \citep{Imp95,FS99,VZ13}, we can construct ``time-bounded'' universal adversaries. That is, for time $ T $, there exists a perturbation adversary $ \sfP_T $ running in time $ \poly(T) $ that finds adversarial examples for all adversaries running in time $ T $ or less. This is insufficient to imply one-way functions though.

\paragraph{Public Key Encryption.} 
The two distributions described above have the following public-key encryption flavor: the robust classifier can serve as the decryption algorithm to distinguish between samples from the perturbed distributions $ D_0' , D_1'$. If after seeing enough samples, the learning algorithm can generate fresh samples from the two unperturbed distributions $ D_0, D_1 $ then we also have an encryption algorithm: to encrypt a bit $ b $, first sample from the distribution $ D_b $ and run the perturbation adversary $ \sfP $ to generate the encryption of the bit. To decrypt, use the robust classifier. 

There are two key ingredients missing: (1) The encryption algorithm $ \sfP $ needs access to fresh samples from the two distributions to encrypt. There are learning tasks where we do not have access to these. (2) The ability to sample the robust classifier along with descriptions of the learning tasks. This might not be feasible, especially when the tasks are not chosen, but supplied by nature. 
%
%
%
%
%

\ifconf
\else

\paragraph{Acknowledgments.} We would like to thank Shafi Goldwasser and Nadia Heninger for discussions regarding inversion of the (noisy) BBS PRG.

\bibliographystyle{alpha}
\newcommand{\etalchar}[1]{$^{#1}$}

\appendix
\fi

\section{A Description of BLPR Example and the Blum-Blum-Shub PRG.}\label{sec:bpr-bbs}

In this section, we describe the BLPR counter-example and the Blum-Blum-Shub pseudorandom generator. 

We start by defining the notion of a trapdoor pseudorandom generator.  A trapdoor pseudorandom generator $ \tprg: \zo^n \rightarrow \zo^{2n}$ is an expanding function whose outputs are indistinguishable from truly random strings. That is, $ \set{\tprg(x) : x \gets \zo^n} \approx_c \set{y: y\gets \zo^{2n}} $.
Furthermore, the function has a trapdoor $ \state $ that allows distinguishing the output of the PRG from random outputs. That is, there exists a distinguisher $ \sfD $ that given the trapdoor,
\[ \prob[x \gets \zo^n]{\sfD(\state, \tprg(x)) = 1} - \prob[y \gets \zo^{2n}]{\sfD(\state, y) = 1} > 0.99\]

Given a trapdoor PRG, the BLPR learning task $ D_0, D_1 $ is the following:
\[ D_0 = \set{(0,\tprg(x)) : x \gets \zo^n} \text{ and, } D_1= \set{(1, y) : y\gets \zo^{2n}} \ .\]

We describe the BBS PRG and its trapdoor property next. 
The Blum-Blum-Shub pseudorandom generator is defined as follows: 

\bigskip
Consider a number $ N = pq $ where $ p, q $ are primes congruent to $ 3\pmod{4} $. 
The seed to the PRG is a random element $ x_0\in \ZZ[N] $.  Let $ {\sf hcb} $ be a hardcore bit\footnote{A function $ {\sf hcb} $ is a hardcore bit of a one-way function $ f $ has the following property, that if given $ y = f(x)  $ for a random $ x $, $ {\sf hcb}(x) $ is pseudorandom. That is, given any algorithm that given $ y = f(x) $ can predict $ {\sf hcb}(x) $, then we can use this algorithm to invert $f $ with non-negligible probability. } of the function $ x \rightarrow x^2 \pmod N$ (eg parity or the most significant bit).
	
	\underline{\sf{BBSPRG}$(x_0, m) $}: 
	\begin{enumerate}[noitemsep,nolistsep]
	\item For $ i \in [1:m] $, 
	\begin{enumerate}[noitemsep,nolistsep]
		\item Set $ x_i = x_{i-1}^2 \pmod N $. 
		\item Set $ y_i = {\sf hcb}(x_i) $
	\end{enumerate}
	\item Output $ y_1, y_2, \dots, y_{m-1}, x_m $. 
	\end{enumerate}

The trapdoor property BLPR refer to construct the robust classifier is the following one: 
In the construction of the PRG, the security does not rely on outputting the last entry ($ x_m $) in its entireity. Though doing so enables the following ``trapdoor'' property: 
\begin{lemma}
	There exists a distinguisher $ \sfD $ that given the factorization of $ N $ can distinguish between the output of the $ \sf{BBSPRG} $ from random strings. That is, 
	\[ \prob[{x_0 \gets \ZZ[N]}]{ \sfD_{p,q}(\sf{BBSPRG}(x_0))} - \prob[y\gets \zo^m]{ \sfD_{p,q}(y)} > 0.99 \]
\end{lemma}
\begin{proof}[Proof Sketch]
The proof relies on the fact that Rabin's one way function $ f(x) = x^2 \bmod N $ is a trapdoor function that can be  efficiently inverted given the factorization of $ N $. Furthermore, the inverse returned is the only square root of $ x^2 $ that is a square itself. Hence the distinguisher does the following: 

\noindent
\begin{minipage}{0.6\textwidth}
	\medskip
	\underline{\sf{D}$(z) $}: 
	\begin{enumerate}[noitemsep,nolistsep]
		\item Interpret the input as $ y_1, y_2, \dots y_{m-1}, x_m $. 
		\item If $ x_m $ is not a square mod $ N $, output $ 0 $. 
		\item Compute $ x_1, x_2, \dots x_{m-1} $ as $ x_i = f^{-1}(x_{i+1}) $. 
		\item If $ y_i = {\sf hcb}(x_i) $ for all $ i $, return $ 1 $, else return $ 0 $.  
	\end{enumerate}
	\medskip
\end{minipage}

Observe that the distinguisher always outputs $ 1 $ on outputs of the PRG. On the other hand, when fed a random string, $ x_m $ is not a square with probability $ 3/4 $ and even when it is a square, the probability of each $ y_i = {\sf hcb}(x_i) $ is exactly $ 1/2 $ independently. Hence the probability that the distinguisher outputs $ 1 $ on a random string is $ \frac{1}{4} \cdot (\frac12)^{m-1} $ which is tiny. 
\end{proof}

Based on this, the BLPR counterexample is the following: 
\begin{blprexample*}
	Let $ N = pq $ where $ p,q $ are random $ n $-bit primes of the form $ 3 \pmod4 $. Let $ m = n^2 $. Define $ D_0, D_1 $ as: 
	\[ D_0 =\set{ (0, {\sf BBSPRG}(x_0) ) : x_0\gets \ZZ[N] } \text{ and,  } D_1: \set{ (1, z) : z  \gets \zo^{m + \log N}} \]
	Then, the learning task has the following properties: 
	(1) The distributions are easy to classify non-robustly. 
	(2) There exists an inefficient robust classifier for $ \eps = \theta(\sqrt n) $. 
	(3) No efficiently learned classifier can classify better than chance. 
	(4) Given the factorization of $ N $, there exists an efficient robust classifier for $ \eps = \theta(\sqrt n) $. 
	
\end{blprexample*}

Properties 1, 2, 3 are true. To the best of our knowledge, 4 is 
not known to be true. As we described earlier, we know of robust classifiers for $ \eps = O(1) $. This leaves us with the following open questions. 

\paragraph{Open Questions.}
\begin{enumerate}
\item Given factorization of $ N $, prove that there exists an efficient robust classifier for $ \eps = \omega(1) $-bits. 
\item (Perturbation Adversary 1) Consider the perturbation adversary that erases the first bit and adds random noise to each bit of the PRG with prob $ 1/\sqrt{n} $. Given the factorization a $ N $, does there exists an efficient robust classifier for this adversary. 
\item (Perturbation Adversary 2) The adversary deletes the last complete entry output by the PRG (i.e., $ x_m $). Given the factorization of $ N $, can we distinguish this PRG from random, when no other error is added. 

Although BBS is a trapdoor PRG, it crucially relies on the fact that $ x_m $, the last value is available completely intact. Without access to this value, BBS is still a PRG but it is not clear how to do the trapdoor decoding. 
\end{enumerate}

As we described earlier, Open Question 3 is a long-standing open question in the computational number theory community \citep{nadia-personal-comm,blogpost}. And Open Question 1 is a harder variant of that question. Finally, Question 2 asks a error correction or decoding question -- given the output of a PRG with random errors, can you recover the original PRG string (even given some trapdoor). We are not aware of any way in which this factorization actually helps decoding under random noise.
%
%
%
%

\end{document}